\documentclass[onefignum,onetabnum]{siamonline220329}
\usepackage{lipsum}
\usepackage{amsfonts}
\usepackage{graphicx}
\usepackage{epstopdf}
\usepackage{algorithmic}
\usepackage{soul}
\usepackage[normalem]{ulem}
\ifpdf
  \DeclareGraphicsExtensions{.eps,.pdf,.png,.jpg}
\else
  \DeclareGraphicsExtensions{.eps}
\fi

\usepackage[inline]{enumitem}
\setlist[enumerate]{leftmargin=.5in}
\setlist[enumerate,1]{label=\textit{\alph*)}}
\setlist[itemize]{leftmargin=.5in}
\newlist{todolist}{itemize}{2}
\setlist[todolist]{label=$\square$}

\newsiamremark{remark}{Remark}
\newsiamremark{hypothesis}{Hypothesis}
\crefname{hypothesis}{Hypothesis}{Hypotheses}
\newsiamthm{claim}{Claim}

\headers{PathProx: A Proximal Gradient Algorithm for DNN}{L. Yang, J. Zhang, J. Shenouda, D. Papailiopoulos, K. Lee, and R. D. Nowak}

\title{PathProx: A Proximal Gradient Algorithm for \\ Weight Decay Regularized Deep Neural Networks
\thanks{\vspace{-1em}\funding{This project is supported by the Institute for Foundations of Data Science (IFDS).} Code is available at: \href{https://github.com/Leiay/PathProx}{\texttt{https://github.com/Leiay/PathProx}}.
}}

\author{Liu Yang\thanks{University of Wisconsin, Madison, USA (\email{liu.yang@wisc.edu}, \email{jifan@cs.wisc.edu}, \email{jshenouda@wisc.edu}, \email{dimitris@papail.io}, \email{kangwook.lee@wisc.edu}, \email{rdnowak@wisc.edu}).} \and Jifan Zhang\footnotemark[2] \and Joseph Shenouda\footnotemark[2] \and Dimitris Papailiopoulos\footnotemark[2] \and Kangwook Lee\footnotemark[2] \and Robert D. Nowak\footnotemark[2]
}

\usepackage{amsopn}

\makeatletter
\newcommand{\oast}{\mathbin{\mathpalette\make@circled\ast}}
\newcommand{\make@circled}[2]{%
  \ooalign{$\m@th#1\smallbigcirc{#1}$\cr\hidewidth$\m@th#1#2$\hidewidth\cr}%
}
\newcommand{\smallbigcirc}[1]{%
  \vcenter{\hbox{\scalebox{0.77778}{$\m@th#1\bigcirc$}}}%
}
\makeatother

\newcommand{\vectorize}[1]{\text{vec}(#1)}
\newcommand{\mc}[1]{\mathcal{#1}}

\newcommand{\bv}{\boldsymbol{v}}
\newcommand{\bw}{\boldsymbol{w}}
\newcommand{\bx}{\boldsymbol{x}}
\newcommand{\by}{\boldsymbol{y}}
\newcommand{\bz}{\boldsymbol{z}}
\newcommand{\bW}{\boldsymbol{W}}
\newcommand{\bX}{\boldsymbol{X}}
\newcommand{\bY}{\boldsymbol{Y}}
\newcommand{\bZ}{\boldsymbol{Z}}
\newcommand{\R}{\mathbb{R}}

\def\floor#1{\lfloor #1 \rfloor}
\def\1{\bm{1}}
\DeclareMathOperator*{\argmin}{arg\,min}
\newcommand{\abnorm}{\textsc{$\ell_2$-Path-Norm }}
\newcommand{\pathprox}{\textsc{PathProx}}

\usepackage{stackengine}
\stackMath
\newcommand\tsup[2][2]{%
 \def\useanchorwidth{T}%
  \ifnum#1>1%
    \stackon[-1.3ex]{\tsup[\numexpr#1-1\relax]{#2}}{\mathchar"307E}%
  \else%
    \stackon[-1ex]{#2}{\mathchar"307E}%
  \fi%
}

\usepackage{mathrsfs,bm}
\usepackage{booktabs}
\usepackage{multirow}
\usepackage{subcaption}
\usepackage{comment}
\ifpdf
\hypersetup{
  pdftitle={An Example Article},
  pdfauthor={D. Doe, P. T. Frank, and J. E. Smith}
}
\fi

\begin{document}

\maketitle

% REQUIRED
\begin{abstract}
Weight decay is one of the most widely used forms of regularization in deep learning, and has been shown to improve generalization and robustness. The optimization objective driving weight decay is a sum of losses plus a term proportional to the sum of squared weights.  This paper argues that stochastic gradient descent (SGD) may be an inefficient algorithm for this objective. For neural networks with ReLU activations, solutions to the weight decay objective are equivalent to those of a different objective in which the regularization term is instead a sum of products of $\ell_2$ (not squared) norms of the input and output weights associated with each ReLU neuron. This alternative {(and effectively equivalent)} regularization suggests a novel proximal gradient algorithm for network training.  Theory and experiments support the new training approach, showing that it can converge much faster to the \emph{sparse} solutions it shares with standard weight decay training.
\end{abstract}

% REQUIRED
\begin{keywords}
Deep Neural Networks, Weight Decay, Regularization, Proximal Method, Sparsity 
\end{keywords}

% REQUIRED
\begin{MSCcodes}
68T05, % Learning and adaptive systems in artificial intelligence [See also 68Q32]
68T20, % Problem solving in the context of artificial intelligence (heuristics, search strategies, etc.)
90C26, % Nonconvex programming, global optimization
47A52, % Linear operators and ill-posed problems, regularization [See also 35R25, 47J06, 65F22, 65J20, 65L08, 65M30, 65R30]
82C32  % Neural nets applied to problems in time-dependent statistical mechanics [See also 68T05, 91E40, 92B20]
\end{MSCcodes}

\section{Introduction}
Weight decay is the most prevalent form of explicit regularization in deep learning, which corresponds to regularizing the sum of squared weights in the model. It has been shown to improve the generalization performance of deep neural networks \cite{krogh1991simple, bartlett1996valid, Zhang2017UnderstandingDL} and even plays a role in making models more robust \cite{galloway2018adversarial, Guo2018SparseDW, pang2021bag}. This paper shows that weight decay regularization can be equivalently and more effectively incorporated into training via shrinkage and thresholding. 
To gain some intuition into the connection between weight decay and thresholding, let us consider a key aspect of most neural networks. Deep {neural network} architectures include many types of processing steps, but the basic neuron or unit is common in {all}.
Consider a single unit of the form $\bv \sigma(\bw^T\bx)$, where $\sigma$ is a fixed activation function and $\bv,\bw$ denote its trainable output and input weights.
This single unit is homogeneous if $\bv \sigma(\bw^T\bx) = \alpha\bv \sigma(\alpha^{-1}(\bw^T\bx))$ for all constants $\alpha>0$. 
Weight decay regularization of this unit corresponds to adding a term proportional to $\frac{1}{2}\big(\|\bw\|_2^2 + \|\bv\|_2^2)$ to the optimization {objective, where we omit the bias term for notational ease.}
Among all the equivalent representations of the unit, it is easy to verify that $\alpha^2 = \|\bw\|_2/\|\bv\|_2$ produces the smallest regularization term by the {inequality of arithmetic and geometric means.}
Thus, at a minimum of the objective we have $\frac{1}{2}\big(\|\bw\|_2^2 + \|\bv\|_2^2) = \|\bw\|_2\|\bv\|_2$.  This simple fact is known \cite{grandvalet1998least, neyshabur2014search}, albeit perhaps not widely.  
{This indicates that the global solutions achieved with either regularizer are equivalent, considering potential rescaling if required.}
The form $\|\bw\|_2\|\bv\|_2$ 
is reminiscent of $\ell_1$-type regularization functions, such as the lasso and group lasso regularizers. {As a result, utilizing the $\|\bw\|_2\|\bv\|_2$ norm as a regularizer will promote sparsity in the resulting solution as an additional outcome.}

In this paper, we propose to replace the weight decay regularization terms with terms of the form $\|\bw\|_2\|\bv\|_2$.  The latter regularizer admits a proximal operation {(coined  \pathprox)} that involves a shrinkage and thresholding step for each homogeneous unit. Theory (\cref{sec:theory}) and experiments (\cref{sec:exp}) show this {leads to faster minimization of the weight decay objective, resulting in solutions with smaller Lipschitz constants and thus increased robustness compared to traditional weight decay training. Furthermore, the utilization of the thresholding operation within \pathprox\ promotes the identification of weight decay solutions that are more sparse. Initial result is shown in \cref{fig:w2v2_vs_wd} and \cref{fig:w2v2_vs_wd_decision_boundary}.}

\begin{figure}[ht]
    \centering
    \includegraphics[width=0.4\textwidth]{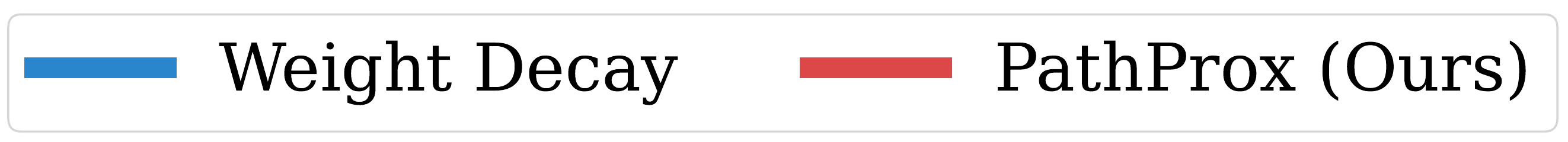}
    \hspace{.39\textwidth}
    \includegraphics[width=0.33\textwidth]{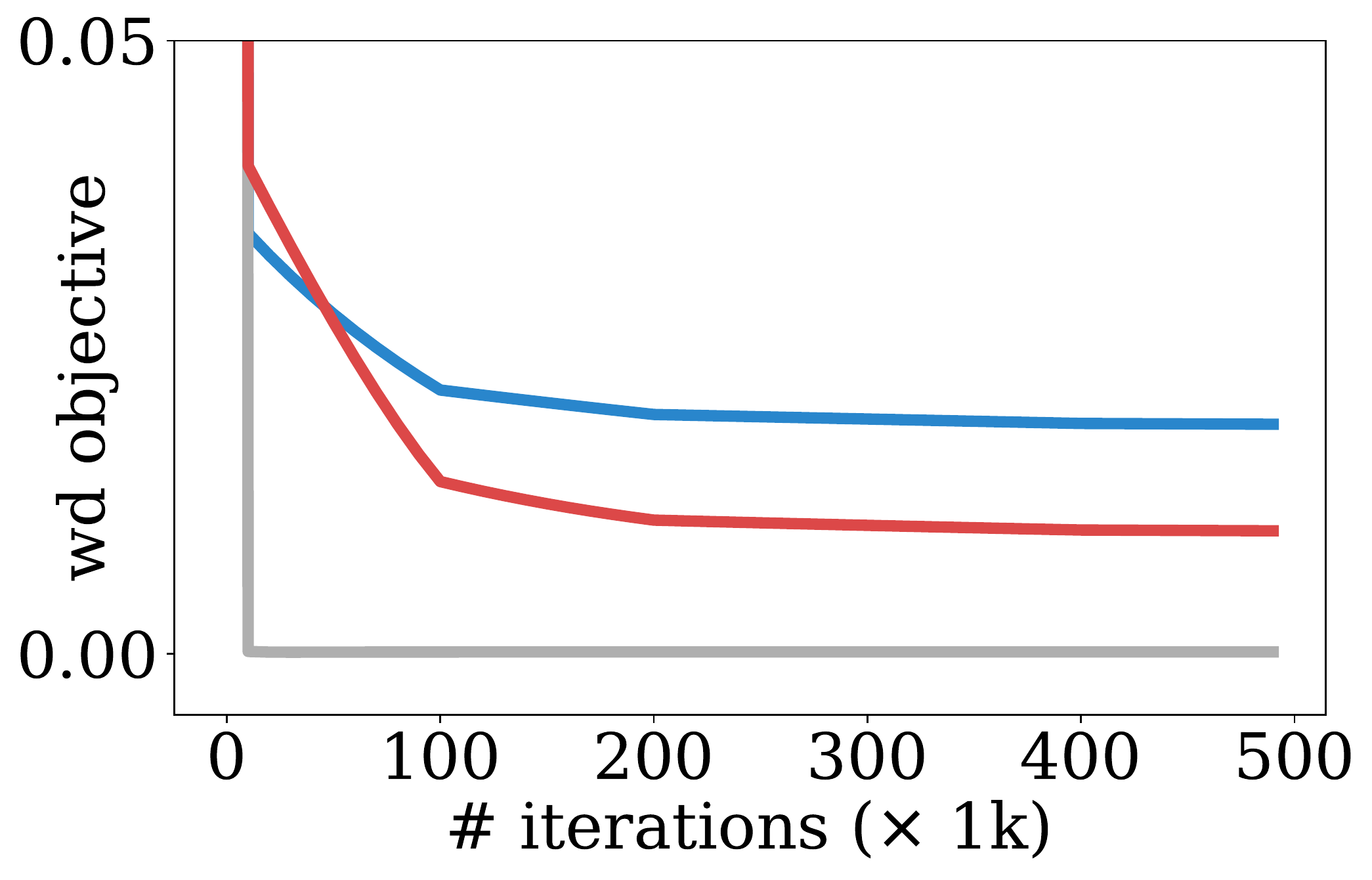}
    \includegraphics[width=0.33\textwidth]{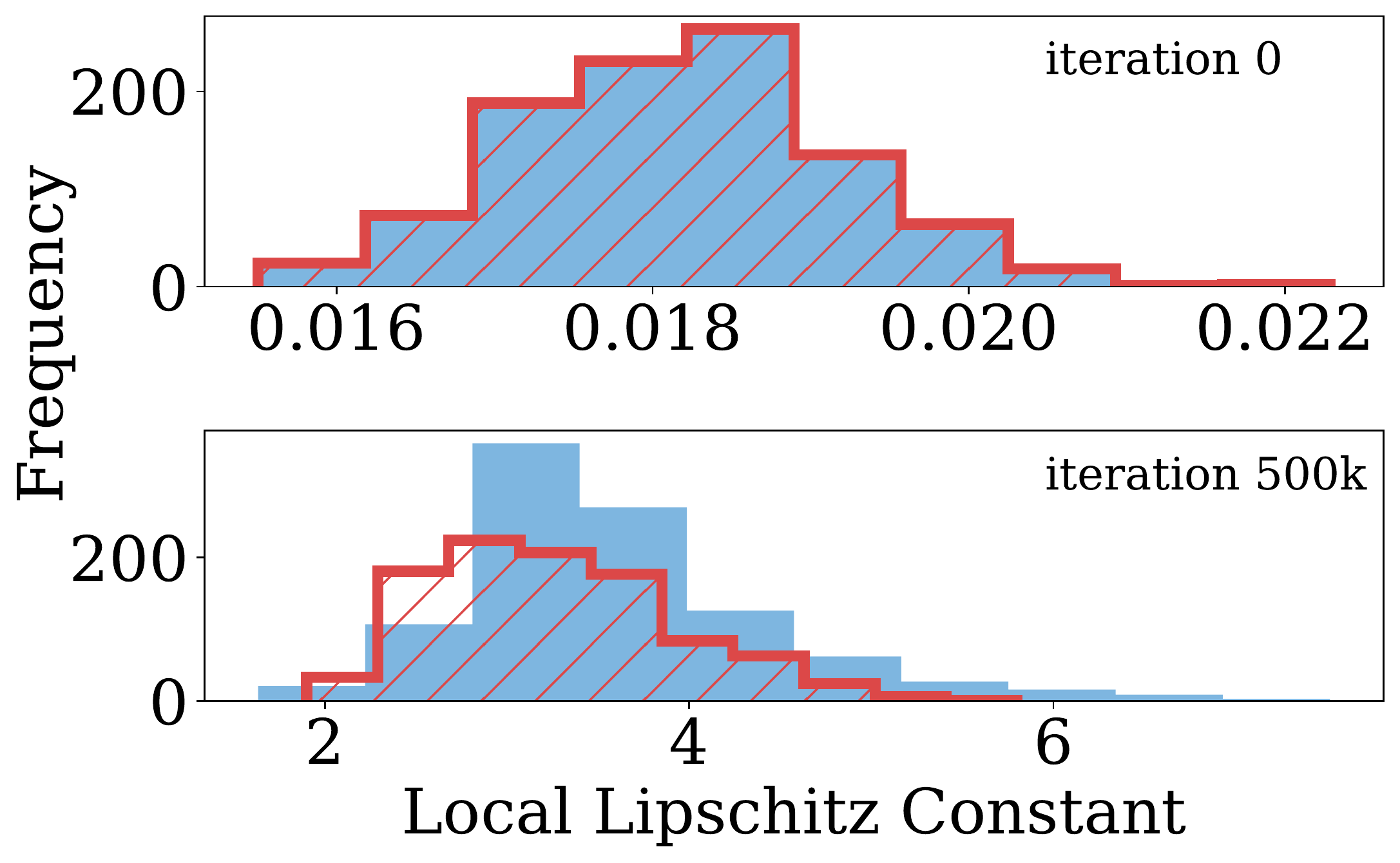}
        \includegraphics[width=0.322\textwidth]{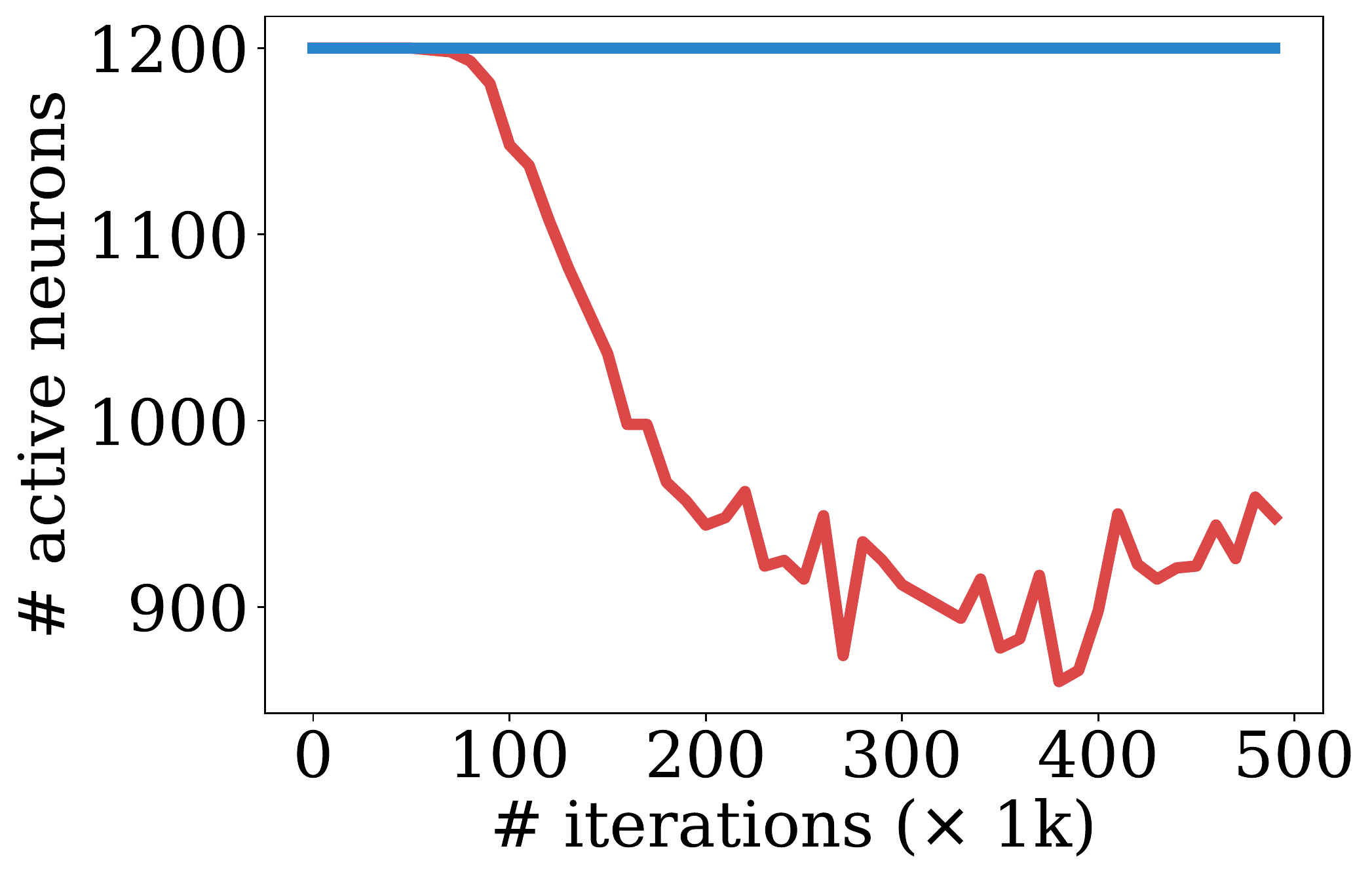}
    \caption{ Comparison between Weight Decay (blue) and 
    \pathprox\ (red) as a function of training {iterations}.
    On (\emph{left}) weight decay objective, with the gray solid line representing the data fidelity loss (both methods fit the data perfectly); (\emph{middle}) {the histogram of local Lipschitz constant on unseen data at the beginning (top) and completion (bottom) of the training process},
    and (\emph{right}) number of active neurons. Please refer to~\cref{appendix:exp_details_for_fig} for detailed experiments setup. 
    }
    \label{fig:w2v2_vs_wd}
\end{figure}

\begin{figure}[ht]
\centering
     \includegraphics[width=0.25\textwidth]{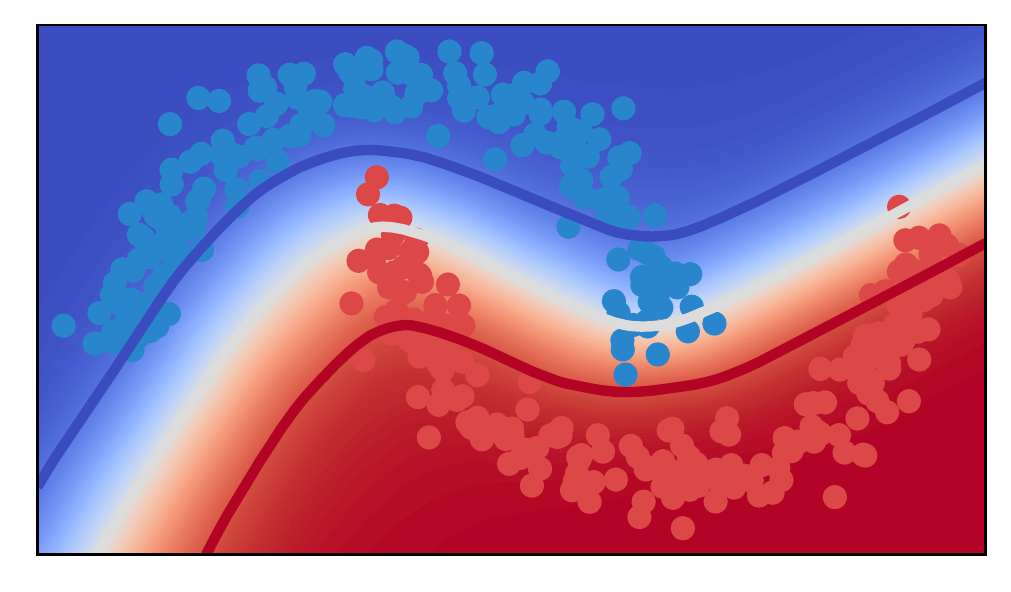}
     \includegraphics[width=0.25\textwidth]{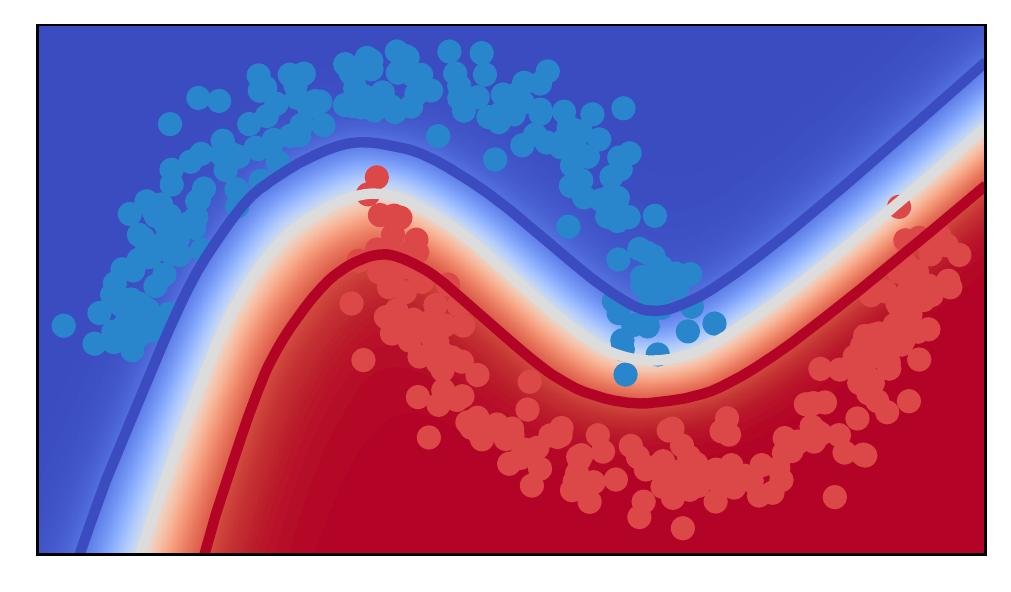}
     \includegraphics[width=0.25\textwidth]{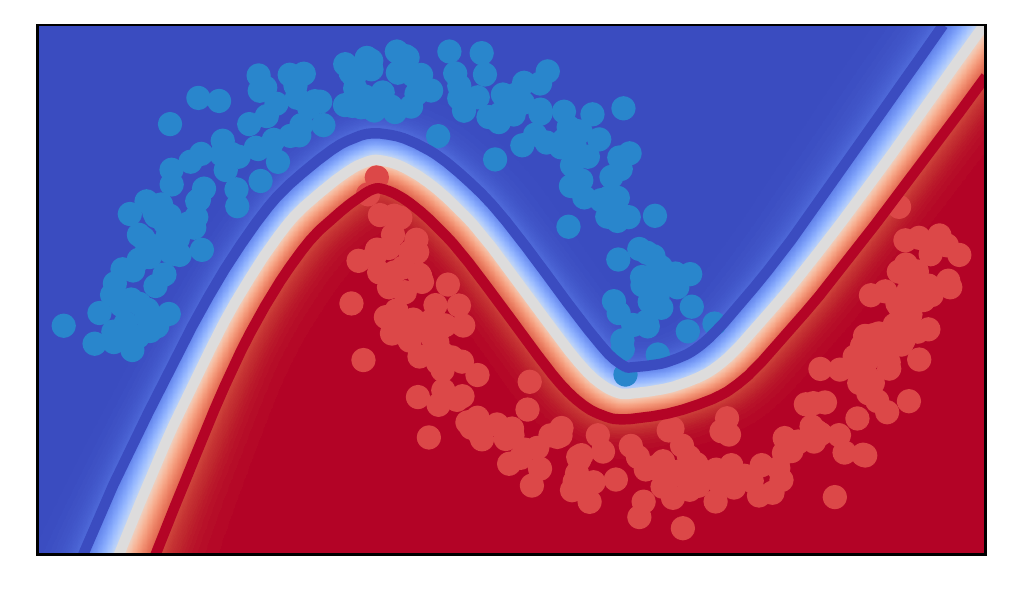}
     % \vspace{-1em}
        \caption{ Decision Boundary of a network trained for binary classifier with Weight Decay (\emph{left}) and 
        \pathprox\ 
        (\emph{middle}) after the same number of training {iterations}. 
        {
        The classifier outputs $[p_0, p_1]$ with $p_0, p_1 \geq 0$ and $p_0 + p_1 = 1$, and predicts class blue if $p_0 \geq p_1$, red otherwise.
        In the figure, the decision boundary is depicted by a white line, indicating the point where the classifier output $p_0 = p_1 = 0.5$. The blue line corresponds to regions where the classifier predicts $p_0 = 0.9$, while the red line corresponds to $p_0 = 0.1$.
        }
        {While both algorithms eventually converge to the same function (\emph{right}), our method learns the boundary \emph{much} faster.}
        Please refer to \cref{appendix:exp_details_for_fig} for details on {the definition of decision boundary} and also the experiment setup.
        }
        \label{fig:w2v2_vs_wd_decision_boundary}
        
\end{figure}

\section{Related Works}

Lines of work have focused on theoretically analyzing equivalent forms of weight decay: \cite{grandvalet1998least} was the first work to highlight an equivalence between weight decay and lasso on the output weights of a neural network. Later \cite{neyshabur2014search,Neyshabur2015NormBasedCC} proved that for a shallow, \emph{single output} network, training the network with squared $\ell_2$ regularization on all the weights (i.e. weight decay) is equivalent to minimizing $\| \bw \|_2  \, |v|$ for each  homogeneous unit. 
This connection was used in \cite{savarese2019infinite,ongie2019function,parhi2022kinds, parhi2021banach, parhi2020role} to study the function space associated with neural networks trained with weight decay. 
For \emph{multi-output} networks training with weight decay is equivalent to minimizing $\|\bw\|_{2}\|\bv\|_{2}$ for each homogeneous unit, \cite{shenouda2023vector} used this connection to characterize the space of functions associated with vector-valued and deep neural networks trained with weight decay.
{More recently, a line of work~\cite{Pilanci2020NeuralNA,Wang2022TheHC,Mishkin2022FastCO,ergen2021convex} has shown how optimizing shallow ReLU neural network can be reduced to a convex program with a group lasso regularizer to the output weights. Here, we also utilize the connection between weight decay and group lasso, but propose a proximal gradient method that instead directly operates on the non-convex objective.}
 Our work is the first to practically exploit these equivalent forms of weight decay by employing proximal methods.

A number of works have theoretically shown that sparse solutions to the weight decay objective exist~\cite{Jacot2022FeatureLI, Jacot2022ImplicitBO, shenouda2023vector, Boursier2023PenalisingTB}.
However, the traditional approach of training deep neural networks with gradient descent and weight decay seldomly produces such sparse solution.
In contrast, our proposed algorithm is based on a sparsity-encouraging proximal gradient algorithm that tends to find sparse solutions to the original weight decay objective. 
This is analogous to the superiority of proximal gradient over pure (sub)gradient methods for $\ell_1$ regularization in linear models (e.g., the lasso problem) \cite{figueiredo2003algorithm,Wright2008SparseRB}. 
Other recent works have proposed related forms of regularization and argued that they find better structured solutions. For instance in \cite{neyshabur2015path, liu2020improve} they utilize the homogeneity of ReLU neural networks to develop regularizers that empirically lead to solutions that generalize better and are more robust \cite{Jiang2020FantasticGM, Dziugaite2020InSO}. 
These alternative regularizers are rarely used in practice, instead we provide a more efficient algorithm for finding solutions to the widely used weight decay objective.

Training regularized neural networks via proximal methods has been employed for the purposes of training quantized \cite{bai2018proxquant,hou2016loss, huang2021training,Yang2020ProxSGD:} or sparse neural networks \cite{fu2022exploring, yoon2017combined, bungert2021bregman,Chen2020OrthantBP}. Specifically, in \cite{fu2022exploring, yoon2017combined} proximal methods were shown to be effective at learning \textit{structurally sparse} networks. In \cite{Latorre2020EfficientPM} they employ a proximal gradient-type algorithm for 1-path-norm  where they focus on the $\|\bw\|_1 \|\bv\|_1$ norm of a homogeneous unit $(\bw, \bv)$ in shallow networks. More general algorithms have been proposed in \cite{Yang2020ProxSGD:} for training neural networks with any non-smooth regularizers and \cite{yun2020general} proposed a framework for second-order stochastic proximal methods on neural networks with non-smooth or non-convex regularizers. 
Based on the rescaling equivalence on homogeneous units, \cite{stock2018equinormalization} also proposed an algorithm to iteratively minimize the weight decay objective though they take inspiration from the Sinkhorn-Knopp algorithm and do not employ proximal methods.
From a more theoretical point of view~\cite{davis2020stochastic} proved that proximal sub-gradient methods are guaranteed to converge to first-order stationary points when used in training deep neural networks. Furthermore, their analysis can be applied directly to our algorithm to guarantee converge.

\section{Training Neural Networks with Weight Decay}\label{sec:theory}

Let $\bW$ denote the weights of a multi-layer neural network. 
Most standard ``training" algorithms fit neural networks to data by minimizing an objective of the form 
\begin{equation}\label{eq:weight_decay_objective}
  F_\lambda(\bW) \ := \ L(\bW) + \frac{\lambda}{2}\,  R(\bW)  
\end{equation}
where $L(\bW)$ is a loss function on the training data, $R(\bW)$ is the sum of squared weights, and $\lambda \geq 0$. When minimized by gradient descent methods, $R(\bW)$ leads to the common practice known as ``weight decay", {which corresponds to a shrinkage operation on the weights after taking the gradient step: $\bW \leftarrow (1 - \lambda \eta)\bW$, where $\eta$ is the stepsize.}  We will call $F_\lambda(\bW)$ the \emph{weight decay objective} and $\lambda$ the \emph{weight decay parameter}.

The neural network may have a general architecture (fully connected, convolutional, etc) and may involve many types of units and operations (e.g., nonlinear activation functions, pooling/subsampling, etc).  This paper focuses on those units in the architecture that are \emph{homogeneous}.
\begin{definition}[Homogeneous Function]
  A function $\sigma$ is homogeneous if it satisfies $\sigma(\alpha x) = \alpha \sigma(x)$ for any $\alpha>0$.    
\end{definition}
For example, the popular Rectified Linear Unit (ReLU), Leaky ReLU and PReLU are homogeneous.  Consider a neuron with input weights $\bw\in\R^p$ and output weights $\bv\in\R^q$.  The neuron produces the mapping $\bx \mapsto \bv \, \sigma(\bw^T\bx)$.  Because $\sigma$ is homogeneous, $\bv \, \sigma(\bw^T\bx) = \alpha \bv \, \sigma(\alpha^{-1}\bw^T\bx)$, for every $\alpha>0$.  
The following {\bf Neural Balance Theorem} provides an important characterization of representations with the minimum sum of squared weights.
\begin{theorem}\label{thm:neuron_balance} (Neural Balance Theorem)
  Let $f$ be a function represented by a neural network and consider a representation of $f$ with the minimum sum of squared weights.   Then the weights satisfy the following \emph{balancing constraints}. Let $\bw$ and $\bv$ denote the input and output weights of any homogeneous unit in this representation. Then $\|\bw\|_2 = \|\bv\|_2.$
 \label{thm:balance}
\end{theorem}
%\vspace{-2\intextsep}
\begin{proof}
  Assume there exists a representation $f$ with minimum sum of squared weights, but does not satisfy the constraint for a certain unit. Because the unit is homogeneous, its input and output weights, $\bw$ and $\bv$,  can be scaled by $\alpha>0$ and $1/\alpha$, respectively, without changing the function. The solution to the optimization $\min_{\alpha>0} \|\alpha\, \bw\|_2^2 + \|\alpha^{-1} \bv\|_2^2$ 
 is $\alpha = \sqrt{\| \bv\|_2/\|\bw\|_2}$.
 Thus, we can rescale the input and output weights to meet the constraint while preserving $f$ yet reducing the sum of squared weights, contradicting the beginning assumption in the proof.
\end{proof}

\begin{remark}
Versions of Neural Balance Theorem (NBT) and its consequences have been discussed in the literature \cite{grandvalet1998least,neyshabur2015path,Neyshabur2015NormBasedCC,savarese2019infinite,ongie2019function,parhi2021banach,parhi2022kinds,Jacot2022ImplicitBO}, but usually in the setting of fully connected ReLU architectures. In {\cite{kunin2021neural} they empirically demonstrate the balance theorem and provide theoretical understanding from the perspective of learning dynamics.}
We note here that NBT holds for any architecture (fully connected, convolutional, pooling layers, etc.) and every homogeneous unit in the architecture.
{Examples of homogeneous units in  fully connected layer and convolutional layer are depicted in~\cref{fig:one_homogeneous}. The details of homogeneous units in different architectures are presented in~\cref{appendix:homo_units_everywhere}.}
\end{remark}
\begin{figure}[h!]
    \centering
    \includegraphics[height=1.1in]{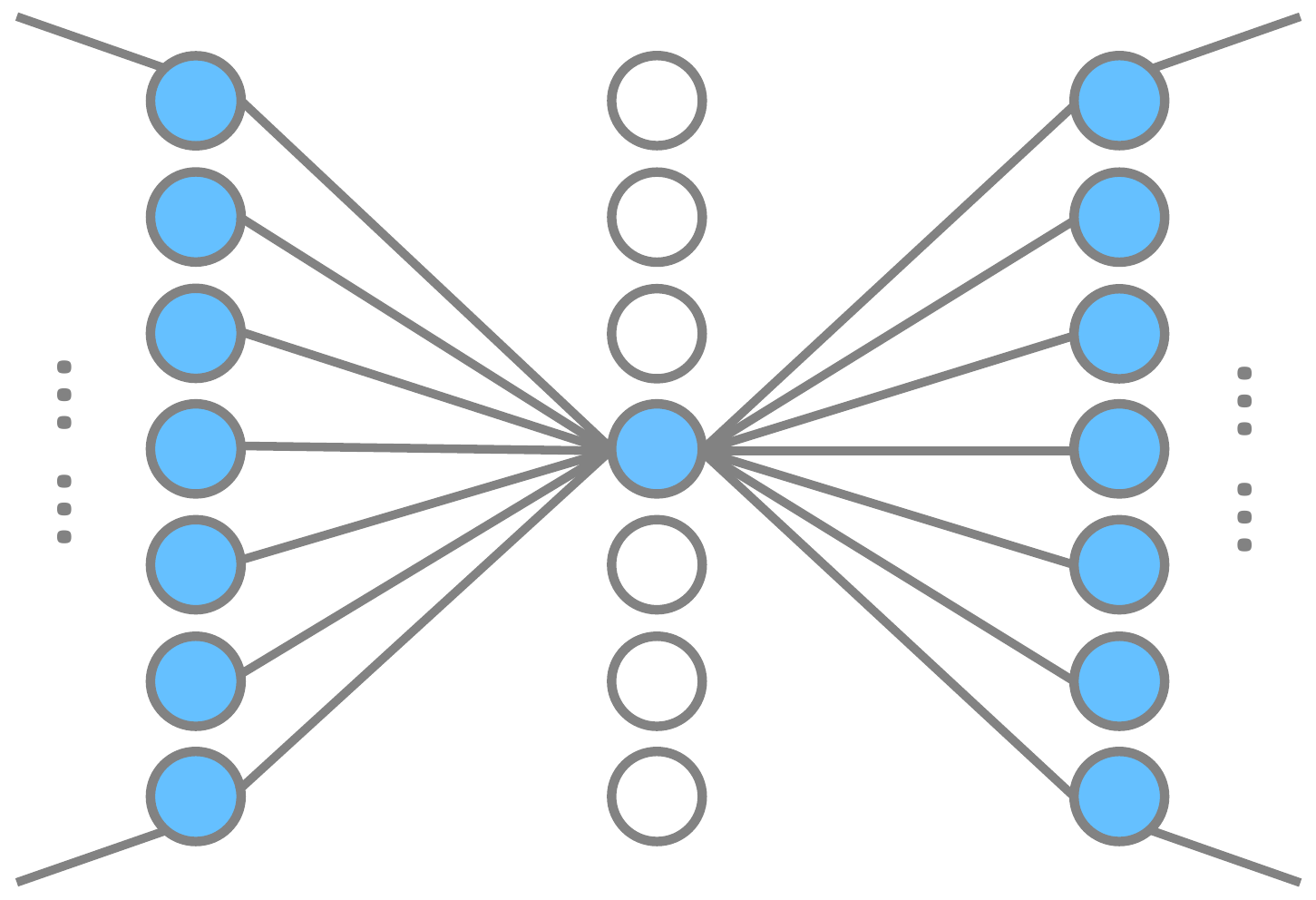}
    \hspace{0.3em}
    \includegraphics[height=1.1in]{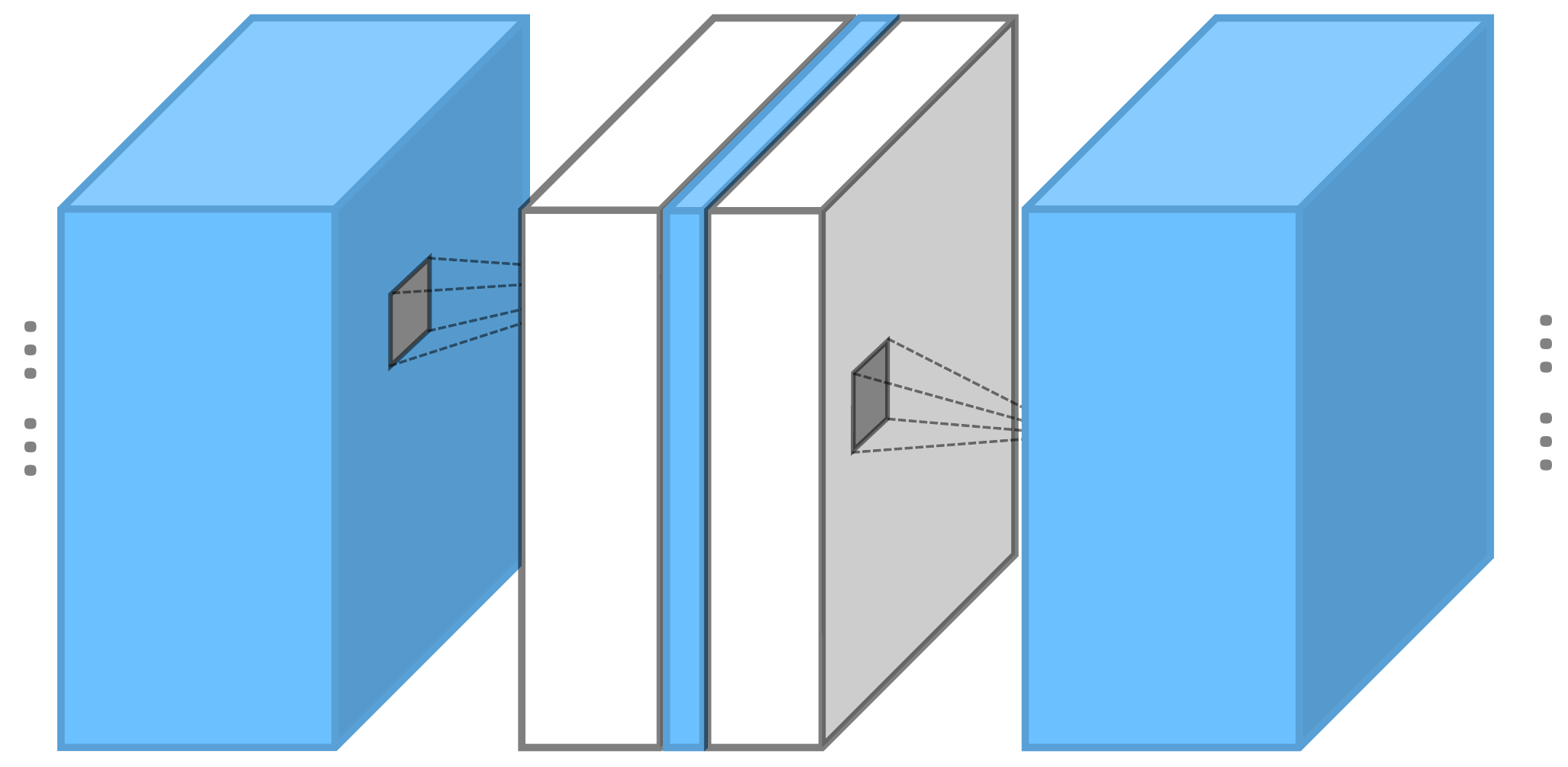}
    \caption{Examples of one-homogeneous unit  in multi-layer perceptron (MLP) (\emph{left}), and convolutional neural network (CNN) (\emph{right}). For MLP, one-homogeneous unit is each neuron; for CNN, one-homogeneous unit is each channel.}
    \label{fig:one_homogeneous}
    % \vspace{-1.5\intextsep}
\end{figure}

\Cref{thm:balance} tells us that the norms of the input and output weights of each homogeneous unit must be equal to each other at a minimum of $F_\lambda$. 
To illustrate a key implication of this, consider a neural network with  $L$ layers of homogeneous {units}.  Let $\bw_{i,k}$ and $\bv_{i,k}$ denote the input and output weights of the $i$th {homogeneous unit} in the $k$th layer.  We use $\bW$ to denote the collection of all weights in all layers. \Cref{thm:balance} shows that a global minimizer of the weight decay objective $F_\lambda$ with any $\lambda>0$ must satisfy the balancing constraints of the theorem.  At a global minimum the $i$th {unit} in the $k$th layer contributes the term $\frac{1}{2}\big(\|\bw_{i,k}\|_2^2+ \|\bv_{i,k}\|_2^2\big) =  \|\bw_{i,k}\|_2  \|\bv_{i,k}\|_2$ to the overall sum of squares $R(\bW)$.  Therefore, if $\bW$ is a global minimizer of $F_\lambda$, then
% \begin{eqnarray}
\begin{equation}
R(\bW) \ = \ \frac{1}{2}\sum^{n_1}_{i=1} \|\bw_{i,1}\|_2^2+ \frac{1}{2}\sum^{n_L}_{i=1} \|\bv_{i,L}\|_2^2+ \sum_{k=1}^L \sum^{n_{k}}_{i=1} \|\bw_{i,k}\|_2 \,  \|\bv_{i,k}\|_2
\label{eq:global_min}    
\end{equation}
% \end{eqnarray}
$n_{k}$ is the number of {homogeneous units} in the $k$th layer. 
The expression above accounts for the fact that input weights of a {unit} in layer $j$ involve the output weights of {unit} in the preceding layer $j-1$. So the weights of internal layers appear twice in the expression above, whereas weights in the input and output layers do not. 
This shows that increasing the weight decay parameter $\lambda$ penalizes the average of the norm-products $\|\bw\|_2  \|\bv\|_2$.  This observation provides {remarkable} insight into the effects of weight decay regularization. 
Let $\eta(\bx):=\bv\, \sigma(\bw^T\bx)$ be a neuron with input weights $\bw\in \R^p$ and output weights $\bv\in \R^q$, and assume that $\sigma$ is $1$-Lipschitz, {(e.g. ReLU)}. Then, by the Cauchy-Schwartz inequality, for all $\bx,\bx'\in\R^m$,
$$\|\eta(\bx)-\eta(\bx') \|_2 \ \leq \ \|\bw\|_2 \, \|\bv\|_2 \, \|\bx-\bx'\|_2 \ , $$
which shows $\|\bw\|_2\|\bv\|_2 $ is a bound on the Lipschitz constant of $\eta$.  Thus, weight decay regularization encourages solutions in which the individual unit functions have small Lipschitz constants on average, a property known to be related to generalization and robustness \cite{bubeck2021law}.

\section{A New Algorithm for Minimizing the Weight Decay Objective}\label{sec:algo}
In this section, we propose a new neural network training algorithm. It exploits the fact that a global minimizer to the weight decay objective is related to a sum of norm-products, as shown in~\cref{eq:global_min}.  The following expression for the sum of squared weights will be helpful in deriving the new algorithm.
Again consider a neural network with homogeneous {units} and $L$ layers,
{and let $\bw_{i,k}$ and $\bv_{i,k}$ be the input and output weights of the $i$th unit in the $k$th layer, with $n_k$ be the total number of homogeneous units in $k$th layer.}
Since the inputs weights of a {unit} in layer $k$ involve the output weights of {units} in the preceding layer $k-1$, we will {restate the weight decay objective which} associates the weights with odd number layers:
\begin{equation}\label{eqn:wd_obj}
    R(\bW) \ = \sum_{j=1}^{\lfloor L/2\rfloor}\sum_{i=1}^{n_{2j-1}} \big(\|\bw_{i,2j-1}\|_2^2+ \|\bv_{i,2j-1}\|_2^2\big) \ + \ c\sum_{i=1}^{n_L} \|\bv_{i,L}\|_2^2
\end{equation}
here $c=0$ if $L$ is even and $c=1$ if $L$ is odd. Each weight appears only once in the expression above, which is convenient for optimization. 
However, \cref{thm:balance} implies that for any $\mathbf{W}$ that minimizes $F_\lambda$, 
{we have that $(\|\bw_{i,2j-1}\|_2^2 + \|\bv_{i,2j-1}\|_2^2) = 2\|\bw_{i,2j-1}\|_2\, \|\bv_{i,2j-1}\|_2$ for each {unit}. This fact leads to the following theorem.} 

\begin{theorem}
For any weights $\bW$ let 
\begin{equation}\label{eqn:pn_obj}
    \widetilde R(\bW) := \sum_{j=1}^{\lfloor L/2\rfloor}\sum_i^{n_{2j-1}} \|\bw_{i,2j-1}\|_2\, \|\bv_{i,2j-1}\|_2 \ + \ \frac{c}{2}\sum_i^{n_L} \|\bv_{i,L}\|_2^2.
\end{equation}
Then the solutions to 
$\min_{\bW} L(\bW)+\frac{\lambda}{2} R(\bW)$ and $\min_{\bW} L(\bW)+\lambda \widetilde R(\bW)$ are equivalent.  Specifically, a minimizer of the first optimization is a minimizer of the second, and a minimizer of the second minimizes the first (after possibly rescaling the weights such that $\|\bw_{i,2j-1}\|_2  = \|\bv_{i,2j-1}\|_2$ for {all homogeneous units}).
\label{thm:path-norm}
\end{theorem}
\begin{proof}
Suppose that $\widehat \bW$ is a solution to $\min_{\bW} L(\bW)+\lambda \widetilde R(\bW)$, but there exists a $\bW$ such that $$L(\widehat \bW)+\frac{\lambda}{2}R(\widehat \bW)> L(\bW)+\frac{\lambda}{2}  R(\bW) \ . $$
\cref{thm:balance} shows that $R(\bW) = 2 \widetilde R(\bW)$, which contradicts the claim that $\widehat\bW$ is a solution to $\min_{\bW} L(\bW)+\lambda \widetilde R(\bW)$.
Next let $\bW$ be a solution to $\min_{\bW} L(\bW)+\frac{\lambda}{2} R(\bW)$, but suppose it does not minimize $L(\bW)+\lambda \widetilde R(\bW)$.  Then there exists a $\widehat \bW$ such that $$L(\bW)+\lambda \widetilde R(\bW)> L(\widehat\bW)+\lambda \widetilde R(\widehat\bW) \ . $$
If necessary, rescale the weights $\widehat \bW$ so that $\|\widehat\bw_{i,2j-1}\|_2  = \|\widehat\bv_{i,2j-1}\|_2$ for each term in $\widetilde R(\bW)$; this does not affect the value of $L(\widehat \bW)$.  Then $R(\widehat\bW)/2=\widetilde R(\widehat\bW)$, which shows that $L(\bW)+\frac{\lambda}{2}  R(\bW)> L(\widehat\bW)+\frac{\lambda}{2} R(\widehat\bW)$, contradicting the assumption that $\bW$ is a solution to the first optimization. 
\end{proof}
{We will refer to the product $\|\bw_{i,j}\|_2\, \|\bv_{i,j}\|_2$ as \abnorm of {the homogeneous unit}. And we will call the objective 
}
\begin{equation}\label{eq:path_norm_obj}
G_\lambda(\bW) \ := \ L(\bW) \ + \ \lambda \widetilde R(\bW)
\end{equation}
the \emph{path-norm objective}.

From \Cref{thm:path-norm} we have that minimizing $G_\lambda$ is equivalent to minimizing the weight decay objective $F_\lambda$.
Our new neural network training algorithm is designed to minimize this $G_\lambda$.  The key observation is that the product terms $\|\bw\|_2\, \|\bv\|_2$ are non-smooth, which means that minimizers may be sparse. 
In fact, as $\lambda$ increases fewer and fewer terms (units) will be nonzero.  This is remarkable, since the sparsity of solutions is not apparent from a cursory inspection of the original weight decay objective, as indicated in~\cref{fig:w2v2_vs_wd} (\emph{right}).
The product terms are reminiscent of \emph{group lasso} regularization terms. 
Proximal gradient descent methods have been widely applied to this (group) lasso type of regularization schemes in the linear/convex cases \cite{yuan2006model,hanson1988comparing,hastie2015statistical,friedman2010regularization}. 
{However, due to the scaling equivalence between $\bw$ and $\bv$, the product term $\|\bw\|_2\|\bv\|_2$ is not amenable to separation via the proximal operator. As a workaround, we propose the imposition of a unit norm constraint on $\bw$. The resulting constrained optimization problem is shown to be equivalent to the original problem in \cref{lemma:equiv_constraint_normalization}.}
{
\begin{lemma}\label{lemma:equiv_constraint_normalization}
The solution sets of
\begin{equation}\label{eq:path_norm_reg_prob}
\min_{\bW} G_{\lambda}(\bW) =  L(\bW)+\lambda \left( \sum_{j=1}^{\lfloor L/2\rfloor}\sum_i^{n_{2j-1}} \|\bw_{i,2j-1}\|_2\, \|\bv_{i,2j-1}\|_2 \ + \ \frac{c}{2}\sum_i^{n_L} \|\bv_{i,L}\|_2^2\right)    
\end{equation}
and
\begin{equation}\label{eq:path_norm_reg_prob2}
\min_{\bW} \widetilde{G_{\lambda}}(\bW) =  L(\bW)+\lambda \left( \sum_{j=1}^{\lfloor L/2\rfloor}\sum_i^{n_{2j-1}} I_{\mathbb{S}}(\bw_{i,2j-1}) \|\bv_{i,2j-1}\|_2 \ + \ \frac{c}{2}\sum_i^{n_{L}} \|\bv_{i,L}\|_2^2 \right)    
\end{equation}
are equivalent. Specifically, any solution to \cref{eq:path_norm_reg_prob2} is a solution to \cref{eq:path_norm_reg_prob}, and any solution to \cref{eq:path_norm_reg_prob} is a solution to \cref{eq:path_norm_reg_prob2} (after rescaling the weights such that $\|\bw_{i,2j-1}\|_2 = 1$ for {all neurons}).
Here $I_{\mathbb{S}}(\bw)$ is defined as:
\begin{footnotesize}$
\begin{cases}
1 & \text{if } \|\bw\|_2=1\\
\infty & \text{o.w.}
\end{cases}.
$\end{footnotesize}
\end{lemma}
\begin{proof}
    We will first start by proving any solution to \cref{eq:path_norm_reg_prob2} is a solution to \cref{eq:path_norm_reg_prob}.
     Suppose $\widetilde{\bW}$ is a solution to \cref{eq:path_norm_reg_prob2}. Then $\widetilde{G_{\lambda}}(\widetilde{\bW}) = {G_{\lambda}}(\widetilde{\bW})$. Now suppose there exists a $\bW$ such that ${G_{\lambda}}(\widetilde{\bW}) > {G_{\lambda}}({\bW})$. We can rescale $\bW$ to be $\bW'$ such that for each homogeneous units: ${\bw}'_{i,2j-1} = \frac{{\bw}_{i,2j-1}}{\lVert {\bw}_{i,2j-1} \rVert}$, and ${\bv}'_{i,2j-1} = \lVert {\bw}_{i,2j-1} \rVert \cdot {\bv}_{i,2j-1}$. Then ${G_{\lambda}}({\bW}) = {G_{\lambda}}({\bW}') = \widetilde{G_{\lambda}}({\bW}') < \widetilde{G_{\lambda}}(\widetilde{\bW})$. This contradict with the assumption that $\widetilde{\bW}$ is the optimal solution for \cref{eq:path_norm_reg_prob2}. Therefore any solution to \cref{eq:path_norm_reg_prob2} is a solution to \cref{eq:path_norm_reg_prob}.
    %%%%%
    
    We now prove any solution to \cref{eq:path_norm_reg_prob} is a solution to \cref{eq:path_norm_reg_prob2} after rescaling the weights as above. Let $\widehat{\bW}$ be an optimal solution to \cref{eq:path_norm_reg_prob}, 
    and $\widehat{\bW}'$ be the rescaled version
    such that ${G_{\lambda}}(\widehat{\bW}) = {G_{\lambda}}(\widehat{\bW}') = \widetilde{G_{\lambda}}(\widehat{\bW}')$.
    Now suppose there exists $\bW$ where 
    $\widetilde{G_{\lambda}}(\bW) < \widetilde{G_{\lambda}}(\widehat{\bW}')$. By construction, $\widetilde{G_{\lambda}}(\bW) = {G_{\lambda}}(\bW)$, therefore we have that
    ${G}_{\lambda}(\mathbf{W}) < {G}_{\lambda}(\widehat{\mathbf{W}})$. This contradicts the hypothesis that $\widehat{\bW}$ is the optimal solution to \cref{eq:path_norm_reg_prob}. Thus any solution to \cref{eq:path_norm_reg_prob} is a solution to \cref{eq:path_norm_reg_prob2} (after rescaling). 
\end{proof}}

{
With the equivalence in the optimization stated above, we propose our proximal algorithm \pathprox\ for \abnorm in~\cref{alg:proximal}.
In the algorithm, we project each $\bw$ to the unit sphere, and apply the standard proximal gradient step for group lasso on $\bv$. Following the analysis in \cite{davis2020stochastic}, with mild assumptions on the \emph{boundedness of our neural network} and proper choice of step sizes, \pathprox\ provably converges:
\begin{theorem}[informal convergence analysis for \pathprox]
    Let $\left\{\bW_k\right\}_{k \geq 1}$ be the iterates produced by \pathprox. Then almost surely, every limit point $\bW^{\star}$ of the iterates $\left\{\bW_k\right\}_{k \geq 1}$ is a stationary point for the problem \cref{eq:path_norm_obj}, and a stationary point for the problem \cref{eqn:wd_obj} after rescaling the weights.
\end{theorem}
}
{The details of the convergence analysis are presented in \cref{sec:convergence_analysis}.} The algorithm is presented as a full batch gradient {update}, but can be easily modified to SGD/mini-batch style algorithms. 
As indicated in \cref{thm:path-norm}, we may have weights associated to non-homogeneous units: $\{\bv_{i, L}\}_{i=1}^{n_L}$. Standard weight decay is applied to these weights.
\Cref{alg:proximal} can also be applied to convolutional neural networks, where we treat each channel of a given convolutional layer as a one-homogeneous unit. \cref{appendix:homo_units_everywhere} discusses this situation and others in detail.

{Besides \cref{thm:neuron_balance}, the homogeneity between layers implies for any $j$, $k$, the total \abnorm for $j$-th layer should equal to the total \abnorm for $k$-th layer at minimum weight decay representation.
\cref{alg:proximal} doesn't explicitly encourage this equalization of \abnorm across layers. To resolve this, 
we also apply a layer-wise balancing procedure (described in~\cref{appendix:layer_balance}) after each gradient step to ensure the total \abnorm each layer is equalized at every iteration.
} 

\begin{algorithm}\small
\begin{algorithmic}
\STATE \textbf{Input: } loss functions $L$, learning rate $\gamma > 0$, weight decay parameter $\lambda > 0$, total number of iterations $T$.

\FOR{$t = 1, 2, ..., T$}
    \FOR{$j = 1, 2, ..., \lfloor L/2\rfloor$}
        \STATE For each homogeneous unit $i$ in ($2j-1$, $2j$) coupled layer:
        \STATE\textbf{Update homogeneous input weights:}

        \STATE update based on batch gradient: $\by \leftarrow \bw^{t-1}_{i, 2j-1}  - \gamma \frac{\partial L(\bW)}{\partial \bw^{t-1}_{i, 2j-1}} \rvert_{\bW^{t-1}}$.

        \STATE project to have unit norm: $\bw^{t}_{i, 2j-1} \leftarrow \argmin_{\|\bw\|_2 = 1} \|\bw - \by\|_2 = \frac{\by}{\| \by \|_2}$.
        
        \STATE\textbf{Update homogeneous output weights:}
        
        \STATE update based on batch gradient: $\bz \leftarrow \bv^{t-1}_{i, 2j-1}  - \gamma \frac{\partial L(\bW)}{\partial \bv^{t-1}_{i, 2j-1}} \rvert_{\bW^{t-1}}$.
        
        \STATE apply proximal operator: $\bv^{t}_{i, 2j-1} \leftarrow \text{Prox}_2(\bz)$ with $\text{Prox}_2(\bz)_i = \begin{cases} 0 &||\bz||_{2} \leq \lambda \cdot \gamma \\ \bz_i - \lambda \cdot \gamma \frac{\bz_i}{||\bz||_{2}} & \text{ o.w. } \end{cases}$.
    \ENDFOR
    \STATE Apply layer-wise balance procedure to ensure total \abnorm for each layer is equalized: $\sum_{i}^{n_{2j-1}} \|\bw^{t}_{i, 2j-1}\|_2 \|\bv^{t}_{i, 2j-1}\|_2 = \sum_{i}^{n_{2k-1}} \|\bw^{t}_{i, 2k-1}\|_2 \|\bv^{t}_{i, 2k-1}\|_2,\ \forall j, k$.
\ENDFOR
\end{algorithmic}
\caption{{\pathprox:\ The Proximal Gradient Algorithm for \abnorm as discussed in \cref{thm:path-norm} and \cref{lemma:equiv_constraint_normalization}}}
\label{alg:proximal}
\end{algorithm}

\section{Experimental Results}\label{sec:exp}
In this section, we will present the experimental results to support the following claims: 
\pathprox\ 
provides
% \begin{inparaenum}[a)]
\begin{enumerate*}
\item faster convergence
\item better generalization, and
\item sparser solutions    
\end{enumerate*}
% \end{inparaenum}
than weight decay.

\newcommand{\pruneTaskNumb}[0]{5} %%%%%%%%%%%%%%%%%%%%%%%%%%%%%%%%%%%%%%%%%%%%%%%%%%%%%%%%%%%%%%%%%%%%%%%%%
For simplicity, we use the notation \emph{MLP-$d$-$n$} to represent a fully-connected feedforward network with $d$ layers, and $n$ ReLU neurons in each.
The notation \emph{MLP-$d$-$n$ factorized} represents a modification to the \emph{MLP-$d$-$n$}, where each layer is replaced by two factorized linear layers, with $n$ hidden neurons in each.
We evaluate \pathprox\ on the following tasks:
\textbf{(Task 1)} MNIST~\cite{lecun1998gradient} subset on MLP-3-400 factorized,
\textbf{(Task 2)} MNIST on MLP-6-400,
\textbf{(Task 3)} CIFAR10~\cite{krizhevsky2009learning} on VGG19~\cite{Simonyan2015VeryDC},
\textbf{(Task 4)} SVHN on VGG19,
\textbf{(Task \pruneTaskNumb)} MNIST on MLP-3-800.
For MNIST subset, we randomly subsample 100 images per class. 
Details of the dataset and network will be introduced in the~\cref{appendix:exp_setup}.
The details of applying~\cref{alg:proximal} to convolutional neural networks such as VGG19 are discussed in~\cref{appendix:homo_units_everywhere}.

\subsection{\pathprox\ Minimizes Weight Decay Objective Faster}\label{sec:pathprox_vs_pathsgd}
\begin{figure}
    % \vspace{-1.5\intextsep}
  \centering
  \includegraphics[width=0.4\linewidth]{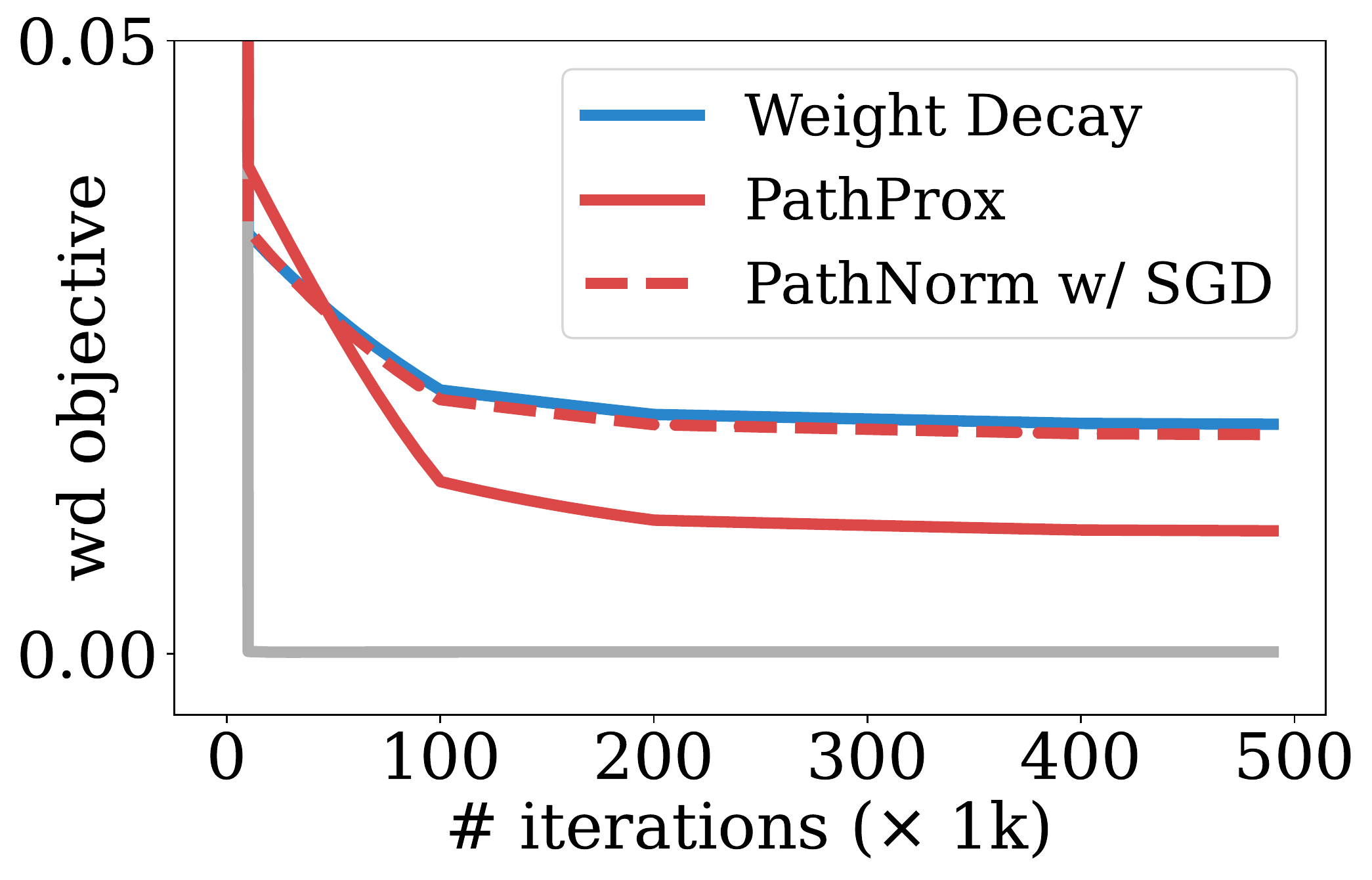}
  \caption{
  Minimizing the weight decay objective by applying \pathprox, SGD on path-norm objective, as well as SGD on weight decay objective. Applying SGD on the path-norm objective and weight decay objective yield a similar convergence rate, but \pathprox\ enjoys faster convergence. The detailed setup is presented in \cref{sec:pathprox_vs_pathsgd}.}
  \label{fig:obj_wd_w2v2_prox_and_sgd}
  % \vspace{-2.7\intextsep}
\end{figure}
{
\pathprox\ applies the proximal gradient algorithm to the path-norm objective and enjoys a faster convergence rate. In this section, we experimentally investigate the individual components of this algorithm to understand the source of its superior performance. Specifically, we verify that simply optimizing the path-norm objective instead of the weight decay objective with SGD yields no significant improvement on the rate of convergence, and thus the faster convergence is attributed to applying the \emph{proximal gradient method} to the path-norm objective.

To do so, we evaluate the weight decay objective~\eqref{eq:weight_decay_objective} on the model trained by 1) applying the proximal gradient algorithm to the path-norm objective~\eqref{eq:path_norm_obj}, which is our proposed \pathprox\ over 2) applying SGD to the path-norm objective, and 3) applying SGD to the weight-decay objective.
We look at the \textbf{(Task 1)}: MNIST subset on MLP-3-400 factorized, with $\lambda=10^{-4}$, and choose the best learning rate for each method (\cref{fig:obj_wd_w2v2_prox_and_sgd}).
Our experiments indicate that the weight decay objective is minimized at a similar rate for 2) and 3), which confirms our discussion in~\cref{thm:path-norm}, that regularizing with the path-norm objective and weight decay objective are equivalent optimization problems. However the proximal method enjoys faster convergence, therefore in the remaining experiments, when regularizing with the \abnorm, we will only consider using \pathprox.
}

\subsection{\pathprox\ Generalizes Better on Test and Corrupted Data}\label{exp:generalization_on_test}
Performance on the unseen dataset, or dataset sampled from other distributions, measures the generalization ability of the model. Following the experimental design in \cite{Zhang2017UnderstandingDL, harutyunyan2020improving, shen2019learning}, the following modification of the data and label are investigated:
\begin{enumerate}[label={\roman*)}]
    \item \textbf{True labels}: the original dataset without modification. \label{item_true_label}
    \item \textbf{Corrupted dataset}: for MNIST, we train on an unmodified dataset, then evaluate the average accuracy across different types of corruptions using the MNIST-C dataset~\cite{Mu2019MNISTCAR}. Similarly, for CIFAR10, we train on an unmodified dataset, then evaluate the average accuracy across different types of corruptions using the CIFAR10-C dataset~\cite{Hendrycks2019BenchmarkingNN}. \label{item_cor_data}
\end{enumerate}

In this section, we present the result of \textbf{(Task 1)}, \textbf{(Task 2)}, and \textbf{(Task 3)} on modification \ref{item_true_label}, \ref{item_cor_data}, and result of \textbf{(Task 4)} on modification \ref{item_true_label} in~\cref{tab:generalization_result_full}. 
For each experiment, we did a grid search on the hyper-parameter choice of $\lambda$ and \emph{learning rate}, and pick the best set of parameters based on the \emph{validation accuracy}. The details of the hyper-parameter search are in~\cref{appendix:exp_setup}. Notice that although we did not explicitly prune the model, our proposed \pathprox\ inherently encourages the sparse structure: in the~\cref{tab:generalization_result_full}, we calculate the structural sparsity of the model, namely the percentage of the active units in the grouped layers, and the results confirm our proximal method can naturally prune the model to have some level of sparsity.
\begin{table}[H]
    \centering
    % \vspace{-1.\intextsep}
    \caption{Generalization results for different modifications on weight decay and \pathprox. Numbers are highlighted if the gap between weight decay and \pathprox\ is at least the sum of both standard errors.}
    \label{tab:generalization_result_full}
    \resizebox{0.9\columnwidth}{!}{%
    \begin{tabular}{clllll}
    \toprule
    \multirow{2}{*}{Task} & \multirow{2}{*}{Modification} & \multicolumn{2}{c}{Weight Decay} & \multicolumn{2}{c}{\pathprox} \\
    \cline{3-6}
    & & Accuracy & Sparsity & Accuracy & Sparsity \\
    \midrule
    \multirow{2}{*}{1} & True labels & $90.87 \pm 0.1$ & 100\% & $\boldsymbol{91.46 \pm 0.11}$ & $99.92 \pm 0.14$\%\\
    & Corrupted data & $64.86 \pm 0.37$ & 100\% & $65.52 \pm 0.65$ &$99.92 \pm 0.14$\%\\
    \midrule
    \multirow{2}{*}{2} & True labels & $98.29 \pm 0.03$ & 100\% & $98.21 \pm 0.07$ & $98.3 \pm 0.78$\%\\
    & Corrupted data & $71.54 \pm 0.22$ & 100\% & $\boldsymbol{72.65 \pm 0.44}$ & $99.95 \pm 0.04$\%\\
    \midrule
    \multirow{2}{*}{3} & True labels & $90.41 \pm 0.1$ & 100\% & $\boldsymbol{90.79 \pm 0.06}$ & $48.29 \pm 10.68$\%\\
    & Corrupted data & $60.78 \pm 0.17$ & 100\% & $60.71 \pm 0.26$ & $60.34 \pm 7.69$ \%\\
    \midrule
    \multirow{1}{*}{4} & True labels & $94.68 \pm 0.12$ & 100\% & $\boldsymbol{95.51 \pm 0.12}$ & $44.02 \pm 14.03$\%\\
    \bottomrule
    \end{tabular}
    }
    % \vspace{-\intextsep}
\end{table}

\subsection{\pathprox\ Finds Sparse Solutions Faster}\label{appendix:structural_pruning}

As the proximal algorithm operates on each one-homogeneous unit (instead of the weight parameter), in this section we focus on the structural sparsity of the model.

\begin{figure}[ht]
% \vspace{-\intextsep}
    \centering
    \includegraphics[width=0.55\linewidth]{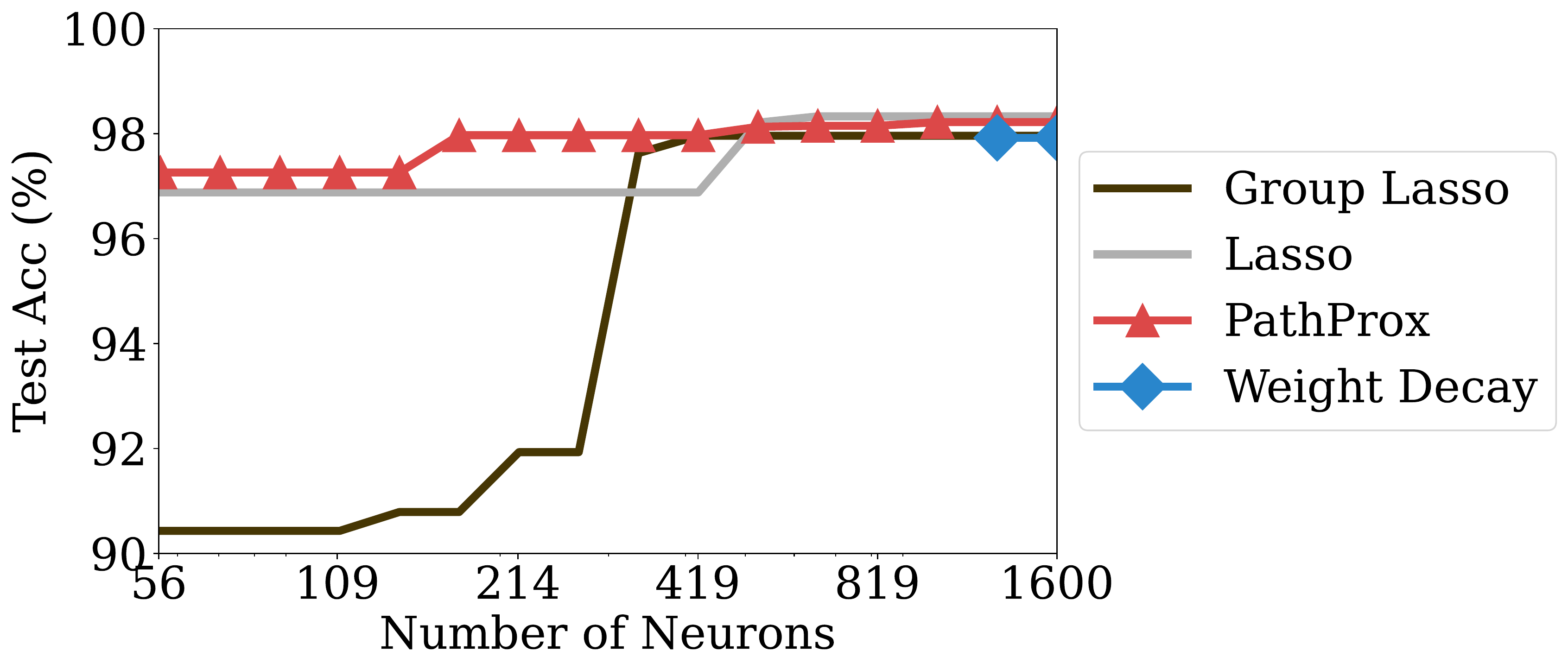}
    \caption{Sparse solution found by 1) weight decay, 2) \pathprox, 3) lasso, and 4) group lasso. {The x-axis is shown in logarithmic scale for clarity.} We highlight the sparse solution of weight decay and \pathprox, which solve the equivalent objective function. However, weight decay could not find the solution as sparse as \pathprox.}
    \label{fig:sparsity}
% \vspace{-\intextsep}
\end{figure}
{While a sparse solution exists for the weight decay objective, training with stochastic gradient descent (SGD) does not consistently lead to such a solution. In contrast, our proposed \pathprox\ method encourages sparsity in the weight decay objective through the inclusion of a thresholding step in the proximal update.
In this section, we compare the sparse solution found by \pathprox, with other prevalent sparsity-inducing regularization: lasso~\cite{Tibshirani1996RegressionSA} and group lasso~\cite{yuan2006model}, on \textbf{(Task \pruneTaskNumb)}.}
{For the latter two regularization term, we simply apply SGD on the corresponding objective}. 
For each experiment, we first prune the model in each training iteration (prune one unit if either its input vector $\bw$ or output vector $\bv$ is zeroed out). After training and pruning for $30000$ iterations, we take the model checkpoint from iterations $\{ 5000, 10000, \cdots, 30000\}$, set the unit to be inactive if $\| \bw \|_2 \|\bv\|_2 < 10^{-5}$, and then train this sparse model for another $10000$ iterations. Since different $\lambda$ and \emph{learning rate} may lead to different levels of sparsity, we try with \emph{learning rate} in $\{ 0.01, 0.03 \}$, and $\lambda \in \{0.0001, 0.001, 0.003, 0.01\}$ for all method.

Performance of the sparse solution is presented in~\cref{fig:sparsity}. For each sparsity level $s$, we present the best solution with sparsity $<s$. Weight decay only finds solutions of sparsity $>64\%$. 
{On the other hand, \pathprox, which minimizes the weight decay equivalent objective, finds sparse solutions that are better than the solutions found by a na\"ive application of SGD on lasso or group lasso.}

{
\subsection{Effects of Weight Normalization} 
In effect, constraining $\bw$ to have norm $1$ is a reparameterization~\cite{Salimans2016WeightNA, Laarhoven2017L2RV} of the optimization problem. 
Here we investigate if the performance gains of \pathprox \ are solely attributable to this reparameterization.
With \textbf{(Task 3)}: CIFAR 10 on VGG19 and modification \ref{item_true_label}: true labels, we evaluate the performance when training with weight decay, and keeping the $\| \bw \|_2$ to be $1$ by projecting the $\bw$ back to the unit sphere after each gradient step. 
In the result~\cref{tab:ablation_study_weight_normalization}, we found a minor digression in accuracy when combining weight normalization with weight decay. Thus, we conclude that the gains in test accuracy observed with \pathprox \ are not from weight normalization and reparameterization.

\begin{table}[H]
    \centering
    % \vspace{-1.\intextsep}
    \caption{Effects of Weight Normalization (WN) on \textbf{(Task 3)} and Modification \ref{item_true_label}.}
    \label{tab:ablation_study_weight_normalization}
    \resizebox{1.\columnwidth}{!}{%
    \begin{tabular}{c|l|cc|cc|cc}
    \toprule
    \multirow{2}{*}{Task} & \multirow{2}{*}{Modification} & \multicolumn{2}{c|}{Weight Decay} & \multicolumn{2}{c|}{Weight Decay (w/ WN)} & \multicolumn{2}{c}{\pathprox(w/ WN by default)}\\ 
    \cline{3-8}
    & & Accuracy & Sparsity & Accuracy & Sparsity & Accuracy & Sparsity \\
    \midrule
    3 & True labels & $90.41 \pm 0.1$ & 100\% & $90.31 \pm 0.05$ & 100\%  & $\boldsymbol{90.79 \pm 0.06}$ & $48.29 \pm 10.68$\%\\
    \bottomrule
    \end{tabular}
    }
    % \vspace{-\intextsep}
\end{table}
}

\section{Conclusion and Future Work} \label{sec:future_work}
This work shows that our proposed \pathprox\ offers advantages in neural network training compared to standard weight decay.  There are several directions for possible future work. One avenue would be to investigate alternative formulations of the proximal gradient method that treat all homogeneous units in the same manner (rather than grouping weights into disjoint sets). Another is developing an adaptive learning rate procedure like those used in other proximal gradient methods.

\section*{Acknowledgments}
The authors would like to thank Stephen Wright, Dmitriy Drusvyatskiy and Ahmet Alacaoglu for helpful discussions regarding the convergence of our algorithm.

\appendix

\section{Homogeneous Units are Everywhere}\label{appendix:homo_units_everywhere}
In the paper, we mainly discuss the theorem and algorithm related to the multi-layer perceptron, but the identification of homogeneous units is not limited to multi-layer perceptron. 
In this section, we formally identify the homogeneous units in common neural network architectures: multi-layer perceptron (MLP) and convolutional neural network (CNN), and include the extension of \cref{thm:balance} to the convolutional neural network.

\subsection{Homogeneous Units in Multi-layer Perceptron (MLP)}
A multi-layer perceptron $f_{\text{MLP}}(\bx; W): \mc{X} \rightarrow \mc{Y}$ with $L$ linear layers takes the following recursive parameterization
\begin{align*}
    & f_{\text{MLP}}(x; W) = W^{L} \begin{bmatrix} 1 \\ h^L \end{bmatrix} \\
    & h^{k + 1} = \sigma\left(W^k \begin{bmatrix} 1 \\ h^k \end{bmatrix} \right), \,\, \forall k \in [L-1] \quad \text{ and } \quad h^1 = x. 
\end{align*}
Here, the linear layers are parameterized by weights $W = \{W^k \in \R^{n_{k+1} \times n_{k}}\}_{k=1}^L$, where $n_k$ is the dimension of the $(k-1)$-th hidden layer, $n_1$ is the input dimension, and $n_{L+1}$ is the output dimension. The ReLU activation function $\sigma(x) = \max\{0, x\}$ is applied element-wise.

Let $\widetilde{W}$ denote the matrix of $W$ with its first column removed, $W_i$ and $W_{:,i}$ denote the $i$-th row and column of $W$, respectively. Consider for every two layers, we have 
\begin{align*}
    \widetilde{W}^{k+1}\sigma\left(W^k \begin{bmatrix} 1 \\ h^k \end{bmatrix}\right) = \sum_{i=1}^{n_k} \widetilde{W}_{:,i}^{k+1} \sigma\left(W_i^k \begin{bmatrix} 1 \\ h^k \end{bmatrix}\right)
\end{align*}
as part of the computation for $h^{k+2}$. Therefore, for every two consecutive layers with weights $W^{2j-1}$ and $\widetilde{W}^{2j}$ where $j \in \left[\lfloor \frac{L}{2} \rfloor \right]$, we identify $n_{2j - 1}$ homogeneous units — one for each hidden neuron. In our experiments, for regular multi-layer perceptron, we use exactly this coupling scheme, where every two layers are combined and viewed as $n_{2j - 1}$ homogeneous units. 

\subsection{Homogeneous Units in Factorized MLP}
Multi-layer perceptron can be equivalent factorized as shown in~\cite{Wang2021PufferfishCM}.
The linear layers are parameterized by weights $W = \{W^k \in \R^{n_{k+1} \times n_{k}}\}_{k=1}^L$, where $n_k$ is the dimension of the $(k-1)$-th hidden layer. For each $W^k$, we can further factorize it into $W^k = Q^k P^k$ for $k \in \{2, 3, \cdots, L-1\}$, where $P^k \in \R^{(n_{k+1} - 1) \times n_{k}}$, $Q^k \in \R^{n_{k+1} \times (n_{k+1} - 1)}$, and
$$
Q^1 = W^1\quad P^L = W^L
$$
Here the bias term is not required for the factorization. Now let $g^k = P^k \begin{bmatrix} 1 \\ h^k \end{bmatrix} \in \R^{n_{k+1}}$, and let $\widetilde{P}$ denote the matrix of $P$ with its first column removed, then $\widetilde{W} = Q \widetilde{P}$ has the first column removed. We have part of the computation for $h^{k+2}$ to be:

\begin{align*}
    Q^{k+1} \widetilde{P}^{k+1} \sigma\left(Q^k P^k \begin{bmatrix} 1 \\ h^k \end{bmatrix}\right) = Q^{k+1} g^{k+1}
\end{align*}
and
\begin{align*}
    g^{k+1} = \widetilde{P}^{k+1} \sigma\left(Q^k g^k\right) = \sum_{i=1}^{n_{k+1} - 1} \widetilde{P}_{:, i}^{k+1} \sigma(Q_i^k g^k)
\end{align*}
where $Q_i$ and $P_{:,i}$ denote the $i$-th row and column of $Q$ and $P$, respectively.
Therefore, for every consecutive layers with factorized weights $Q^k$ and $P^{k+1}$, we identify $n_{k+1}-1$ homogeneous units — one for each hidden neuron. 

Note that for fixed dimensions and number of layers, the class of factorized MLP is equivalent with the class of the original MLPs with only activated layers. Furthermore, with a factorized MLP, we bypass the issue of having to group even number of layers. Instead, we are able to optimize an equivalent class of function and identify a homogeneous unit for each hidden neuron.

\subsection{Homogeneous Units in Convolutional neural networks (CNN)}
Let $\oast$ denote the sliding-window convolutional operator, where for any matrices $\bX, \bY$ and $\bZ$, we have $\bZ_{i,j} = (\bX \oast \bY)_{i,j} = \sum_{a\in [L_1], b\in[L_2]} \bX_{a, b} \bY_{i+a, j+b}$ with $L_1$ and $L_2$ denoting the width and height of $\bX$. A convolutional homogeneous unit then takes the following form with one hidden channel
\begin{align*}
    \mu(\bx) = \left[\bv_j \oast \sigma\left(\sum_{i\in [C_1]} \bw_i \oast \bx_i\right)\right]_{j\in[C_2]}
\end{align*}
where $\bx$, $\bw$ and $\bv$ are three-dimensional tensors and with slight abuse of notation $\sigma(\cdot)$ is applied element-wise. $\bx$ takes the \emph{channel first} notation, where the first dimension indexes number of channels while the last two dimensions indexes width and height. Both $\bx$ and $\bw$ have channel size $C_1$ while $\bv$ has channel size $C_2$.

\subsubsection{Extension of~\cref{thm:neuron_balance} to Convolutional Layers}\label{appendix:extension_to_cnn} 
Below, we restate the equivalence from~\cref{thm:neuron_balance} for convolutional units:

\begin{theorem} (Convolutional Neural Network Balance Theorem)
Let $f$ be a function represented by a neural network and consider a representation of $f$ with the minimum sum of squared weights. Furthermore, let $\vectorize\cdot$ denote the vectorize operator of a tensor. Then the weights satisfy the following \emph{balancing constraints}. Let $\bw$ and $\bv$ denote the input and output weights of any convolutional homogeneous unit $\mu(\bx) = [\bv_j \oast \sigma(\sum_{i\in [C_1]} \bw_i \oast \bx_i)]_{j\in[C_2]}$ in this representation. Then $\|\vectorize{\bw}\|_2 = \|\vectorize{\bv}\|_2.$
\end{theorem}

We also note that one can take a channel-wise homogeneous pooling operation on the hidden channel, i.e., a homogeneous unit of the form $\mu(\bx) = [\bv_j \oast \mc{P}(\sigma(\sum_{i\in [C_1]} \bw_i \oast \bx_i))]_{j\in[C_2]}$, where $\mc{P}$ is some homogeneous pooling function such as max pooling or average pooling. Due to the homogeneity of the pooling layer, the above results also hold for these homogeneous units with pooling layers. 

\subsubsection{Identification of Homogeneous Units in CNN}
A 2D convolutional \emph{backbone network} is a sequence of function composition of convolutional and pooling layers that maps 3D tensors to 3D tensors. The output of a backbone network is usually then flattened and passed through an MLP. We focus on an $L$ layers convolutional backbone network here, which takes the following recursive parameterization
\begin{align*}
    & f_{\text{CNN}}(x; W) = \left[\left(\sum_{j} W_{i,j}^{L} \oast h_j^L \right) \oplus b_i^L\right]_{i\in[n_{L+1}]} \\
    & h^{k + 1} = \left[\mc{P}^k \left(\sigma\left(\left( \sum_j W_{i,j}^k \oast h_j^k\right) \oplus b_i^k\right)\right) \right]_{i\in[n_{k+1}]}, \,\,\forall k \in [L-1] \quad \text{ and } \quad h^1 = x. 
\end{align*}
Here $n_k$ denotes the number of hidden/output channels and $W = \{(W^k, B^k)\}_{k=1}^{L}$ are the weights parameterizing the neural network. Each layer weight $W^k$ is a four-dimensional tensor of dimensions $n_{k+1} \times n_k \times l \times l'$ and $b^k \in \R^{n_{L+1}}$. $\oplus$ is an element-wise addition operator which adds the later scalar argument to the former tensor. $\mc{P}^k$ is a channel-wise pooling layer such as average pooling and max pooling, or the identity function. Each hidden layer $h^k$ is then a three-dimensional tensor of dimensions $n_k \times l \times l'$. By substituting in one more recursive step and for any homogeneous activation function $\sigma$, we get
\begin{align*}
    h^{k + 1} &= \left[\mc{P}^k \left(\sigma\left(\left( \sum_j W_{i,j}^k \oast \mc{P}^{k-1} \left(\sigma\left(\left( \sum_{j'} W_{j,j'}^{k-1} \oast h_{j'}^{k-1} \right) \oplus b_j^k\right)\right) \right) \oplus b_i^k\right)\right) \right]_{i\in[n_{k+1}]} \\
    &= \mc{P}^k \left(\sigma\left(\left[\left( \sum_j W_{i,j}^k \oast \mc{P}^{k-1} \left(\sigma\left(\left( \sum_{j'} W_{j,j'}^{k-1} \oast h_{j'}^{j-1} \right) \oplus b_j^k\right)\right) \right) \oplus b_i^k\right]_{i\in[n_{k+1}]} \right)\right) .
\end{align*}
Therefore, for each $j \in [n_k]$, we have $\left[ W_{i,j}^k \oast \mc{P}^{k-1} \left(\sigma\left(\left( \sum_{j'} W_{j,j'}^{k-1} \oast h_{j'}^{k-1} \right) \oplus b_j^{k-1}\right)\right)\right]_{i\in[n_{k+1}]}$ as part of the computation. With the one-homogeneity of the channel-wise pooling layer $\mc{P}^{k-1}$, we can get an equivalent function by scaling $\left(\alpha W_{j,j'}^{k - 1}, \alpha b_j^{k - 1}\right)$ and $\frac{1}{\alpha} W_{i,j}^k$, we can therefore obtain a homogeneous unit for every hidden channel $j$ of the $k$-th layer $h^k$.

Since the channel-wise pooling layer can be either identity, or max pooling, or average pooling layer, we can group the convolutional layers across max/average pooling layers. Therefore, for every two consecutive convolutional layer (possibly across the pooling layer), with weights $\left(W_{j,j'}^{2k - 1}, b_j^{2k - 1}\right)$ and $W_{i,j}^{2k}$, where $k \in \{1, 2, \cdots, \floor{{L \over 2}}\}$, we identify $n_{2k - 1}$ homogeneous units, one for each channel.

\section{Layer-wise Balancing Procedure}\label{appendix:layer_balance}

In this section, we discuss the layer-wise balancing procedure that enforces balancing constraints across layers. This is motivated by the fact that the sums of \abnorm of every two consecutive layers should be equal at the minimum norm solution. Below, we will first formally show this observation as a corollary of~\cref{thm:neuron_balance} and then present the empirical impact of the layer-wise balancing algorithm.

\subsection{Corollary of~\cref{thm:neuron_balance}: Layer-wise Balancing}
As indicated in~\cref{thm:neuron_balance}, for the minimum sum of squared weight representation, the $\ell_2$ norm of the input vector and output vector of the unit are the same. Thus for the coupling of $(j-1, j)$-th layer with $n_{j-1}$ homogeneous units, we have
$$
\sum_{i=1}^{n_{j-1}} \left\|\bw_{i, j-1} \right\|_2^2 = \sum_{i=1}^{n_{j-1}} \left\|\bv_{i, j-1} \right\|_2^2
$$
and for the coupling of $(j, j+1)$-th layer, we have
$$
\sum_{i=1}^{n_{j}} \left\|\bw_{i, j} \right\|_2^2 = \sum_{i=1}^{n_{j}} \left\|\bv_{i, j+1} \right\|_2^2
$$
which indicate the $j-1$, $j$ and $j+1$-th layer have the same amount of sum of squared weights. Since $j$ is arbitrary, it is easily verified that at the minimum norm representation, each layer share the same amount of sum of squared weights. 
Now consider the $(j, j+1)$-th and $(k, k+1)$-th coupling layer, we have
$$
\frac{1}{2}\sum_{i=1}^{n_j} \left\|\bw_{i, j} \right\|_2^2 + \left\|\bv_{i, j+1} \right\|_2^2 = \frac{1}{2}\sum_{i=1}^{n_k} \left\|\bw_{i, k} \right\|_2^2 + \left\|\bv_{i, k+1} \right\|_2^2
$$
Again, as indicated in~\cref{thm:neuron_balance}, at minimum norm representation, we have $\| \bw \|_2 = \| \bv \|_2$, therefore
$$
\sum_{i=1}^{n_j} \left\|\bw_{i, j} \right\|_2  \left\|\bv_{i, j+1} \right\|_2 = \sum_{i=1}^{n_k} \left\|\bw_{i, k} \right\|_2 \left\|\bv_{i, k+1} \right\|_2
$$
So the sum of the \abnorm per coupling of layers are the same for the minimum norm solution. Our proposed proximal algorithm doesn't naturally enforce this, so we will apply the layer-wise balance procedure along the proximal algorithm.

{
\subsection{Effects of Layer-wise Balancing}\label{sec:ablation_study_layerwise_balance}
In this section, we study the impact of layer-wise balance: on \textbf{(Task 3)} and modification \ref{item_true_label}, we evaluate the performance of \pathprox\ without layer-wise balance, and demonstrate the result in~\cref{tab:ablation_study_layerwise_balance}.
This result suggests Layer-wise Balance is beneficial, but not the essential factor to make our proximal algorithm outperform plain weight decay.
\begin{table}[H]
    \centering
    % \vspace{-1.\intextsep}
    \caption{Effects of Layer-wise Balance (LB) on \textbf{(Task 3)} and Modification \ref{item_true_label}.}
    \label{tab:ablation_study_layerwise_balance}
    \resizebox{1.\columnwidth}{!}{%
    \begin{tabular}{c|l|cc|cc|cc}
    \toprule
    \multirow{2}{*}{Task} & \multirow{2}{*}{Modification} & \multicolumn{2}{c|}{Weight Decay} & \multicolumn{2}{c|}{\pathprox\ (w/ LB by default)} & \multicolumn{2}{c}{\pathprox\  (w/o LB)} \\
    \cline{3-8}
    & & Accuracy & Sparsity & Accuracy & Sparsity & Accuracy & Sparsity \\
    \midrule
    3 & True labels & $90.41 \pm 0.1$ & 100\% & $\boldsymbol{90.79 \pm 0.06}$ & $48.29 \pm 10.68$\% & $90.58 \pm 0.04$ & $40.47 \pm 3.74$\%\\
    \bottomrule
    \end{tabular}
    }
    % \vspace{-\intextsep}
\end{table}
}

{
\section{Convergence Analysis of \cref{alg:proximal}}\label{sec:convergence_analysis}
We utilize the analysis on stochastic proximal subgradient methods from \cite{davis2020stochastic} to prove that any limit point of our algorithm~\cref{alg:proximal} is a first-order stationary limit point. We will omit the details of their result, and focus on how it applies to our setting: 
Without loss of generality, we will just consider the number of layers to be even in the deep neural network. 
With $I_{\mathbb{S}}(\bw)$ defined as:
\begin{footnotesize}$
I_{\mathbb{S}}(\bw) = \begin{cases}
1 & \text{if } \|\bw\|_2=1\\
\infty & \text{o.w.}
\end{cases}
$\end{footnotesize}, recall that we seek to solve
\begin{equation}
\begin{aligned}
\min_{\bW} G_{\lambda}(\bW) &= L(\bW)+\lambda R(\bW) \\
&= L(\bW)+\lambda \left( \sum_{j=1}^{\lfloor L/2\rfloor}\sum_{i=1}^{n_{2j-1}} I_{\mathbb{S}}(\bw_{i, 2j-1}) \|\bv_{i,2j-1}\|_2 \right)
\end{aligned}
\end{equation}
which can be equivalently expressed as the following constrained optimization
\begin{equation}\label{eq:opt_sub_gradient_ours}
\begin{aligned}
    \min _{\bW \in \mathcal{W}} &G_{\lambda} (\bW)= L(\bW)+ \lambda R(\bW) \\
    &\text{where } R(\bW) = \sum_{i=1}^{n} \| \bv_i \|_2
\end{aligned}
\end{equation}
where we simplify the notation by using $i$ to index the homogeneneous units in the network (omitting the notation to indicate the layer of each neuron), and $n = n_1 + n_2 + \cdots + n_L$. $\bW$ is the set of input and output weights $\{\bw_i, \bv_i\}_{i=1}^n$ associated with neuron $i$, and $\mathcal{W}$ is the weight space $\mathbb{S}^{d-1} \times \mathcal{V}^{d}$: $\bw_i \in \mathbb{S}^{d-1}$, and $\bv_i \in \mathcal{V}^{d}$, where $\mathcal{V}^{d}$ is the closed bounded set in $\mathbb{R}^d$ such that $\forall \bv \in \mathcal{V}^{d}$, we have $\| \bv \|_2 \leq C < \infty$. Here we make the assumption that the weights are bounded. We also assume the datasets we are working on are bounded as well.
$L(\bW)$ is a composition of the softmax cross-entropy loss function with the neural network function. 
The main theorem in~\cite{davis2020stochastic} requires some assumptions (which we will verify next) and states that:

\begin{theorem}[Corollary~6.4 in~\cite{davis2020stochastic}]\label{thm:davis_thm}
Assume that our problem set-up satisfies Assumption E in~\cite{davis2020stochastic}  and that $L, R$, and $\mathcal{W}$ are definable in an o-minimal structure. Let $\left\{W_k\right\}_{k \geq 1}$ be the iterates produced by the proximal stochastic subgradient method. Then almost surely, every limit point $W^{\star}$ of the iterates $\left\{W_k\right\}_{k \geq 1}$ converges to a first-order stationary point of \cref{eq:opt_sub_gradient_ours}, i.e. satisfies
$$
0 \in \partial L(W^{\star}) + \lambda \partial R(W^{\star}) + N_{\mathcal{W}}(W^{\star}) 
$$
\end{theorem}
Here $N_{\mathcal{W}}(W^{\star})$ is the Clarke normal cone to $\mathcal{W}$. Specifically, for the $i$th neuron we have, 
\begin{align*}
    N_{\mathbb{S}^{d-1}}(\bw_i^{\star}) = \{c\bw_i^{\star}, c \in \mathbb{R}\}, \quad N_{\mathbb{R}^{d}}(\bv_i^{\star}) = \{\mathbf{0}\}.
\end{align*}

All that remains is to verify that our problem set-up satisfies
Assumption E in~\cite{davis2020stochastic} and that functions $L, R$ and weight space $\mathcal{W}$ are definable on an o-minimal structure.

\begin{corollary}
    The proximal stochastic subgradient method~\cref{alg:proximal} for \cref{eq:opt_sub_gradient_ours} satisfies the  Assumption E in~\cite{davis2020stochastic}.
\end{corollary}
\begin{proof}

Property 1 of Assumption E requires $\mathcal{W}$ to be closed. This holds because both $\mathbb{S}^{d-1}$ and $\mathcal{V}^d$ are complete under the usual Euclidean norm. Property 1 also requires $L$ and $R$ to be locally Lipschitz. By our boundedness assumption on the weights of the network as well as the boundedness on the data, the neural network function is locally Lipschitz. Combining this with the fact that softmax and cross-entropy loss are both locally Lipschitz we can conclude that $L$ is locally Lipschitz since it is simply a composition of these three functions. We also have that $R$ is locally Lipschitz,
\begin{equation}\label{eq:R_is_lipschitz}
\begin{aligned}
    \frac{R(x) - R(z)}{\|x - z\|_2} &\leq \frac{R(x - z)}{\|x - z\|_2} = \frac{\sum_{i \in \mathcal{G}} \|x_i - z_i\|_2}{\|x - z\|_2} \\
    &\leq \frac{\| x - z \|_1}{ \| x - z \|_2} \leq \frac{\sqrt{d}\|x - z\|_2}{\|x - z\|_2} = \sqrt{d}
\end{aligned}
\end{equation}
where $\|\cdot\|_2$ denote the $\ell_2$ norm, and $\|\cdot\|_1$ denote the $\ell_1$ norm, and $\mathcal{G}$ is the group.
Therefore, Property 1 is satisfied. 

Property 2 requires the slope of the secant line of $R$ to be bounded by a bounded function. This is immediately satisfied from \eqref{eq:R_is_lipschitz}.
Property 3 requires the stepsize to be nonnegative, squared-summable but not summable. This can be achieved by a proper choice of stepsizes.
Property 4 demands the weights are bounded, which is trivially satisfied by our boundedness assumption. 
Property 5 and 6 require the first and second moments of the subgradient norm to be bounded by a bounded function. This also immediately follows from the boundedness assumption of the weights and the data.
\end{proof}

We note that the assumption we impose on the boundedness of the weights is reasonable with proper initialization and optimization. In all of our experiments, following the standard deep neural network weight initialization~\cite{he2015delving} and proper step-sizes, we never observe the weights growing unbounded.

Now we show that $L$ and $R$ are definable on an o-minimal structure, and the weight space $\mathcal{W}$ are definable. Since the composite of two definable functions is definable~\cite[Theorem 2.3]{loi2010lecture}, it suffices to show the following:
\begin{corollary}
    The following components are definable: 1) ReLU function $\max(0, x)$; 2) softmax function $s(\boldsymbol{x}) = \frac{\exp(\boldsymbol{x})}{\mathbf{1}^T \exp(\boldsymbol{x})}$; 3) cross-entropy loss $\ell(\boldsymbol{y}, \boldsymbol{z}) = -\log(\boldsymbol{z}_i)$; 4) group lasso norm $R(\{\bv_i\}) = \sum_{i} \| \bv_i \|_2$; 5) The weight space for $\bw_i$: $\mathbb{S}^{d-1}$; 6) The weight space for $\bv_i$: $\mathcal{V}^d$.
\end{corollary}
\begin{proof}
    First, \cite{davis2020stochastic} shows $\max(0, t)$, $\exp(t)$ and $\log(t)$ are definable, and the bounded norm space $\mathcal{V}^{d}$ is definable. So 1), 3) and 6) are immediately satisfied. 
    As pointed out in~\cite{davis2020stochastic} there exists an o-minimal structure containing all semialgebraic functions by the result in~\cite{wilkie1996model}, therefore to show that a function $v$ is definable on an o-minimal structure it suffices to show that the function is semialgebraic which means that the \emph{graph} of the function is a semialgebraic set. For a function $v(x)$, its graph is the set 
    $$
    \hat{G}(v) = \{ (x,v(x)) : x \in X\}
    $$
    this set is semialgebraic if it can be expressed as,
    $$
    \{ u \in \hat{G}(v): p_i(u) \leq 0, \quad \text{for }i=1, \cdots, \ell \}
    $$
    where $p_i$ are polynomials. In the following, we will prove each of the functions is definable by constructing the semialgebraic set of the graph of the functions.
    \noindent
    For 2), we can re-write this as a composite of function: $s(\boldsymbol{x}) = \frac{\exp(\bx)}{\mathbf{1}^T \exp(\bx)} = f \circ \exp(\bx)$, where $f(\bz) = \frac{\bz}{\mathbf{1}^T\bz}$. For $f$, 
    $$
    \begin{aligned}
        \hat{G}(f) &= \left\{\left(\bz, \frac{\bz}{\mathbf{1}^T\bz}\right) \in \R^{d+} \times \R^{d+}, \forall \bz \in \R^{d+} \right\} \\
        &= \left\{(\bz, \by) \text{ s.t. } \by_i = \frac{\bz_i}{\sum_j \bz_j} \forall i \in [d] \right\} \\
        &= \left\{(\bz, \by) \text{ s.t. } \bz_i - \by_i \left(\sum_j \bz_j\right) = 0, \forall i \in [d] \right\} \\
    \end{aligned}
    $$
    Thus, $f$ is definable, and by~\cite{wilkie1996model} $\exp(\bx)$ is also definable, therefore $s(\cdot)$ is definable as it is the composition of two definable functions~\cite[Theorem 2.3]{loi2010lecture}.
    For 4), let $R = S \circ \|\cdot \|_2$, where $S$ is the summation operator. We can show $\| \cdot \|_2$ is definable by:
    $$
    \begin{aligned}
        \hat{G}(\| \cdot \|_2) &= \{(\bx, \|\bx\|_2) \in \R^d \times \R, \forall \bx \in \R^d\} \\
        &= \{(\bx, y) \text{ s.t. } y^2 = \|\bx\|_2^2\} \\
        &= \left\{(\bx, y) \text{ s.t. } y^2 - \sum_{i=1}^d \bx_i^2 = 0 \right\} \\
    \end{aligned}
    $$
    and $S(\cdot)$ is definable by
    $$
    \begin{aligned}
        \hat{G}(S) &= \{(\bx, \mathbf{1}^T\bx) \in \R^K \times \R, \forall \bx \in \R^K \} \\
        &= \left\{(\bx, y) \text{ s.t. } y - \sum_{i=1}^K \bx_i = 0 \right\} \\
    \end{aligned}
    $$
    So $R(\cdot)$ is definable.
    For 5), it reduce to showing $\hat{G}(\| \cdot \|_2)$ is semialgebraic set, which we have already proved.
\end{proof}

} 

\section{Details of Experiments}
\subsection{Experiment Setup} \label{appendix:exp_setup}
In this section we will describe the dataset and network we use in our experiments. All our experiments are conducted on Nvidia 3090 GPUs.

\subsubsection{Datasets}\label{appendix:datasets}
In this work, we demonstrate the performance of \pathprox\ on the following dataset:
\begin{enumerate}[label={\alph*)}]
    \item \textbf{MNIST}~\cite{lecun1998gradient} consist of 10 classes of hand-written digits, each class has 6000 training images and 1000 test images. When training, we randomly split the dataset into 55000 training data, and 5000 validation data. We use the validation data to decide the hyper-parameters. Each image has shape $28 \times 28$. When training, we normalize the input data.
    \item \textbf{MNIST-C}~\cite{Mu2019MNISTCAR} is a robustness benchmark on MNIST dataset, which applies 15 standard corruption to the MNIST dataset, namely
    % \begin{inparaenum}[1)]
    \begin{enumerate*}
    \item shot noise,
    \item impluse noise,
    \item glass blur,
    \item motion blur,
    \item shear,
    \item scale,
    \item rotate,
    \item brightness,
    \item translate,
    \item stripe,
    \item fog,
    \item spatter,
    \item dotted line,
    \item zigzag, and
    \item canny edge.
    % \end{inparaenum} 
    \end{enumerate*}
    When validate our result on MNIST-C dataset, we train on clean MNIST dataset, and pick the best model based on clean validation set, then test its performance on the corrupted MNIST-C dataset to measure the generalization.
    \item \textbf{CIFAR10}~\cite{krizhevsky2009learning} has 10 classes of real images. Each class has 5000 training images and 1000 test images. When training, we randomly split 45000 images for training, and 5000 images for validation. Each image has shape $32 \times 32$. When training, we random crop, random horizontal flip, and normalize the input data.
    \item \textbf{CIFAR10-C}~\cite{Hendrycks2019BenchmarkingNN} is a robustness benchmark on CIFAR10 dataset, which applies 18 standard corruption to the CIFAR10 dataset, namely
    \begin{enumerate*}
    \item Gaussian noise,
    \item shot noise,
    \item impluse noise,
    \item defocus blur,
    \item frosted glass blur,
    \item motion blur,
    \item zoom blur,
    \item snow,
    \item frost,
    \item fog,
    \item brightness,
    \item contrast,
    \item elastic,
    \item pixelate, and
    \item JPEG.
    \end{enumerate*} When validate our result on CIFAR10-C dataset, we train on clean CIFAR10 dataset, and pick the best model based on clean validation set, then test its performnce on the corrupted CIFAR10-C dataset.
    \item \textbf{SVHN}~\cite{Netzer2011ReadingDI} is the Street View House Numbers dataset, with 73257 digits for training, 26032 digits for testing. We randomly split 67257 images for training, and 6000 images for validation. Each image has shape $32 \times 32$. When training, we normalize the input data. We didn't use the additional dataset to boost the performance.
\end{enumerate}

\subsubsection{Models}
In this work, we demonstrate the performance of \pathprox\ on the following models:
\begin{enumerate}[label={\alph*)}]
    \item \textbf{MLP-$d$-$n$} model consist of $d$ fully connected layers, each with $n$ neurons. Details of the \emph{MLP-$d$-$n$} architecture is shown in~\cref{arch:mlp-typed}.
    \item \textbf{MLP-$d$-$n$ factorized} model consist of $d$ fully connected layers, each with $n$ neurons. For each layer, we factorize it into two linear layers, with hidden neurons to be $n$ as well. Details of the \emph{MLP-$d$-$n$ factorized} architecture is shown in~\cref{arch:mlp-typed}.
    \item \textbf{VGG19} is introduced in~\cite{Simonyan2015VeryDC}, which is widely used for computer vision task. Instead of 3 fully-connected layer as the classifier, to apply VGG19 on CIFAR10, we use 1 fully-connected layer instead\footnote{Code adapted from \url{https://github.com/kuangliu/pytorch-cifar/blob/master/models/vgg.py}}. There are 16 convolutional layers. After every two or four convolutional layers, there follows a max-pooling layer to reduce the feature map size by half.
\end{enumerate}

In~\cref{exp:generalization_on_test}, we demonstrate the generalization result of weight decay and \pathprox\ on:
\begin{itemize}[{}]
    \item \textbf{(Task 1)} MNIST subset on MLP-3-400 factorized
    \item \textbf{(Task 2)} MNIST on MLP-6-400
    \item \textbf{(Task 3)} CIFAR10 on VGG19
    \item \textbf{(Task 4)} SVHN on VGG19
\end{itemize}

In~\cref{appendix:structural_pruning}, we evaluate the ability to obtain sparse solution of weight decay, lasso and group lasso on
\begin{itemize}[{}]
    \item \textbf{(Task \pruneTaskNumb)} MNIST on MLP-3-800
\end{itemize}
When evaluating the \pathprox, we need to couple the layers into groups. We choose to factorize it, and thus evaluate the algorithm on:
\begin{itemize}[{}]
    \item \textbf{(Task \pruneTaskNumb)} MNIST on MLP-3-800 factorized
\end{itemize}

\begin{table}[h!]  
\footnotesize
\caption{The MLP architecture used in the experiments. For MLP-3-400, MLP-6-400, and MLP-3-800 factorized, we group each coupling layer together. For MLP-3-800, we refer to each layer as a group.}
\label{arch:mlp-typed}
\centering
\begin{tabular}{ccccc} %ResNet-20 
\toprule
\textbf{Parameter} & \textbf{MLP-3-400 factorized} & \textbf{MLP-6-400} & \textbf{MLP-3-800} & \textbf{MLP-3-800 factorized} \\
\midrule
\multirow{4}{*}{Group 1} & 784$\times$400 & 784$\times$400 & 784$\times$800 & 784$\times$800\\
& ReLU & ReLU & ReLU & ReLU \\
& 400$\times$400 & 400$\times$400 & & 800$\times$800\\
&  & ReLU &  & \\
\midrule
\multirow{4}{*}{Group 2} & 400$\times$400 & 400$\times$400 & 800$\times$800 & 800$\times$800\\
& ReLU & ReLU & ReLU & ReLU \\
& 400$\times$400 & 400$\times$400 & & 800$\times$800\\
&  & ReLU & & \\
\midrule
\multirow{4}{*}{Group 3} & 400$\times$400 & 400$\times$400 & 800$\times$800 & 800$\times$800\\
& ReLU & ReLU & ReLU & ReLU \\
& 400$\times$10 & 400$\times$10 & & 800$\times$10 \\
\midrule
Group 4 & & & 800$\times$10 & \\
\bottomrule
\end{tabular}
\end{table}

\subsubsection{Hyper Parameter Choice}

For generalization experiments in~\cref{exp:generalization_on_test}, we did grid search for the learning rate as well as the $\lambda$ as follows:
\begin{enumerate}[{}]
    \item \textbf{(Task 1) \& (Task 2)}:
    \begin{itemize}
        \item \emph{learning rate}: 0.003, 0.01, 0.03, 0.1, 0.3, 0.5
        \item $\lambda$: 0.00001, 0.00003, 0.0001, 0.0003, 0.001, 0.003, 0.01
    \end{itemize}
    \item \textbf{(Task 3)}:
    \begin{itemize}
        \item \emph{learning rate}: 0.01, 0.03, 0.1, 0.3
        \item $\lambda$: 0.00001, 0.00003, 0.0001, 0.0003, 0.001, 0.003, 0.01
    \end{itemize}
    \item \textbf{(Task 4)}:
    \begin{itemize}
        \item \emph{learning rate}: 0.03, 0.1, 0.3, 0.5
        \item $\lambda$: 0.0001, 0.0003, 0.001, 0.003
    \end{itemize}
    \item \textbf{(Task \pruneTaskNumb)}:
    \begin{itemize}
        \item \emph{learning rate}: 0.01, 0.03
        \item $\lambda$: 0.0001, 0.001, 0.003, 0.01
    \end{itemize}
\end{enumerate}

{After searching in the hyper-parameter space, we pick the hyper-parameter with largest validation accuracy. The choice of the hyper-parameter for each task is presented in~\cref{tab:hyper_parameter_choice}.}
\begin{table}[h!]
    \centering
    \caption{Hyper-parameter choice for weight decay and \pathprox\ for the experiments shown in \cref{tab:generalization_result_full}}
    \label{tab:hyper_parameter_choice}
    \resizebox{0.7\columnwidth}{!}{%
    \begin{tabular}{clll}
    \toprule
    Task & Modification & Weight Decay & \pathprox\\ 
    \midrule
    \multirow{2}{*}{1} & True labels & $\lambda=0.001$, lr$=0.3$ & $\lambda=0.0001$, lr$=0.3$ \\
    & Corrupted data & $\lambda=0.001$, lr$=0.3$ & $\lambda=0.0001$, lr$=0.3$\\
    \midrule
    \multirow{2}{*}{2} & True labels & $\lambda=0.00003$, lr$=0.1$ & $\lambda=0.0003$, lr$=0.1$ \\
    & Corrupted data & $\lambda=0.00003$, lr$=0.1$ & $\lambda=0.0003$, lr$=0.1$ \\
    \midrule
    \multirow{2}{*}{3} & True labels & $\lambda=0.001$, lr$=0.1$ & $\lambda=0.001$, lr$=0.1$\\
    & Corrupted data & $\lambda=0.001$, lr$=0.1$ & $\lambda=0.001$, lr$=0.1$\\
    \midrule
    \multirow{1}{*}{4} & True labels & $\lambda=0.001$, lr$=0.3$ & $\lambda=0.003$, lr$=0.1$\\
    \bottomrule
    \end{tabular}
    }
\end{table}

\subsubsection{Standard Error Calculation}
For each task, we run the grid search experiments on one random seed, and pick the set of hyper parameter (\emph{learning rate}$^*$, $\lambda^*$) based on the best validation accuracy. Then for each task, we run with (\emph{learning rate}$^*$, $\lambda^*$) for another three times. Again in each run, the test accuracy is picked based on the best validation accuracy. With four runs, we evaluate the mean and standard error of the experiments, and present the result in the table.

\subsection{Experiments Details for~\cref{fig:w2v2_vs_wd,fig:w2v2_vs_wd_decision_boundary}}\label{appendix:exp_details_for_fig}

\paragraph{\cref{fig:w2v2_vs_wd}} We assess the performance of \textbf{(Task 1)} using $\lambda=0.0001$ and determine the optimal learning rate, which is $0.3$, for both \pathprox\ and weight decay. To ensure successful training, we implement a learning rate decay schedule. Looking at the convergence plot on the left, we observe that while both \pathprox\ and weight decay achieve the same data fidelity loss (indicated by the gray curve) on the training set, \pathprox\ effectively minimizes the weight decay objective at a faster rate. For the histogram in the \emph{middle}, we measure the local Lipschitz constant for unseen data by calculating the spectral norm of the Jacobian (of the model output with respect to the input) on 1000 MNIST test samples. Notably, the model trained with weight decay exhibits a generally larger spectral norm of the Jacobian compared to our approach, indicating that \pathprox\ leads to models with a lower local Lipschitz constant, enhancing their robustness. Finally, the sparsity plot on the right demonstrates that \pathprox\ showcases empirical evidence of finding a solution with a sparse structure, which is an additional advantage stemming from the thresholding operation in the proximal gradient update.

\paragraph{\cref{fig:w2v2_vs_wd_decision_boundary}} The decision boundary depicted in the figure is generated by training a shallow network with a single hidden layer on the given data points. The raw output (or logits) of this network is a two-dimensional vector, denoted as $[v_0\ v_1]$, where $v_i$ belongs to the set of real numbers ($\R$). Subsequently, these $v_i$ values are passed through the softmax function, defined as $p_i = s(v_i) = \frac{\exp(v_i)}{\exp(v_0) + \exp(v_1)}$. By applying the softmax function, the resulting $p_i$ values satisfy the property that $p_0 + p_1 = 1$. 
We draw the \emph{white} line on samples that has $p_0 = p_1 = 0.5$. Let class 0 represent the blue class, and class 1 be the red class, then we draw the \emph{blue} line on samples that has $p_0 = 0.9$, and \emph{red} line on samples with $p_0 = 0.1$. The background color of the figure represents the values of $p_0$ and $p_1$, where red indicates a high value of $p_1$ close to 1, and blue indicates a high value of $p_0$ close to 1.
A smaller region between the blue and red lines indicates that the classifier is more effective in handling dataset outliers. In the figure, when trained with the same number of iterations, \pathprox\ is better at handling the outliers compared to the weight decay with gradient descent.
In detail, the weight decay parameter in this experiment is $\lambda=0.0001$ and \emph{learning rate} (SGD step size) is 0.1. Both algorithms are initialized identically and run for the same number of iterations.

\bibliographystyle{siamplain}
\bibliography{references}

\end{document}